\documentclass[letterpaper]{article} % DO NOT CHANGE THIS
\usepackage{aaai25}  % DO NOT CHANGE THIS
\usepackage{times}  % DO NOT CHANGE THIS
\usepackage{helvet}  % DO NOT CHANGE THIS
\usepackage{courier}  % DO NOT CHANGE THIS
\usepackage[hyphens]{url}  % DO NOT CHANGE THIS
\usepackage{graphicx} % DO NOT CHANGE THIS
\usepackage{natbib}  % DO NOT CHANGE THIS AND DO NOT ADD ANY OPTIONS TO IT
\usepackage{caption} % DO NOT CHANGE THIS AND DO NOT ADD ANY OPTIONS TO IT
\usepackage[ruled,vlined]{algorithm2e}
\usepackage{xcolor}
\usepackage{cite}
\usepackage{amsmath,amssymb,amsfonts}
\usepackage{amsthm}
\usepackage{booktabs}

\newtheorem{definition}{Definition}
\newtheorem{proposition}{Proposition}
\newtheorem{theorem}{Theorem}
\newtheorem{lemma}{Lemma}
\newtheorem{corollary}{Corollary}
\newtheorem{assumption}{Assumption}

\title{
A Communication and Computation Efficient Fully First-order Method for Decentralized Bilevel Optimization}
\author {
     Ming Wen\textsuperscript{\rm 1}\thanks{Email: mwen23@m.fudan.edu.cn}, 
     Chengchang Liu\textsuperscript{\rm 2}, 
     Ahmed M. Abdelmoniem\textsuperscript{\rm 3},
     Yipeng Zhou\textsuperscript{\rm 4},
     Yuedong Xu\textsuperscript{\rm 1}\thanks{Email: ydxu@fudan.edu.cn}
}
\affiliations {
    \textsuperscript{\rm 1}Fudan University,\\
    \textsuperscript{\rm 2}The Chinese University of Hong Kong\\
    \textsuperscript{\rm 3}Queen Mary, University of London\\
    \textsuperscript{\rm 4}Macquarie University\\
}

\usepackage{bibentry}
\begin{document}

\nocopyright
\maketitle

\begin{abstract}
Bilevel optimization, crucial for hyperparameter tuning, meta-learning and reinforcement learning, remains less explored in the decentralized learning paradigm, such as decentralized federated learning (DFL). 
Typically, decentralized bilevel methods rely on both gradients and Hessian matrices to approximate hypergradients of upper-level models. 
However, acquiring and sharing the second-order oracle is compute and communication intensive. % and sharing this information incurs heavy communication overhead. 
To overcome these challenges, this paper introduces a fully first-order decentralized method for decentralized Bilevel optimization, $\text{C}^2$DFB which is both compute- and communicate- efficient.
In  $\text{C}^2$DFB, each learning node optimizes a min-min-max problem to approximate hypergradient by exclusively using gradients information.  
To reduce the traffic load at the inner-loop of solving the lower-level problem, 
 $\text{C}^2$DFB incorporates a lightweight communication protocol for efficiently transmitting compressed residuals of local parameters. % during the inner loops. 
Rigorous theoretical analysis ensures its convergence % of the algorithm,  indicating a first-order oracle calls 
of $\tilde{\mathcal{O}}(\epsilon^{-4})$. 
Experiments on hyperparameter tuning and hyper-representation tasks validate the superiority of $\text{C}^2$DFB across various typologies and heterogeneous data distributions.

\end{abstract}

\section{Introduction}

% {\color{red} Paragraph 1: A description of federated learning and bilevel optimization}

% {\color{red} Paragraph 2: Challenges faced by federated and decentralized optimization: limited computing power and long communication latency}

% {\color{red} Paragraph 3: introduce the problem that we are going to study}

% {\color{red} Paragraph 4: How we solve the problem}

% {\color{red} Paragraph 5: More details}

% {\color{red} Paragraph 6: Summary of contributions}

% {\color{blue}

Bilevel optimization covers a significant category of hierarchical optimization problems.  
Solving such problems involves an upper-level problem that depends on the solution of a lower-level problem. 
This approach has a variety of important applications in machine learning, including hyperparameter optimization \cite{pmlr-v162-gao22j,pmlr-v139-ji21c}, meta-learning  \cite{Qin_2023_CVPR,pmlr-v80-franceschi18a}, and reinforcement learning \cite{doi:10.1137/20M1387341}. 
For instance, when tuning hyperparameters of machine learning models, the model training is considered the lower-level problem, while the selection of hyperparameters is the upper-level problem.

However, solving  bilevel optimization in the decentralized learning paradigm is both compute- and communication-intensive, whereas it receives little research attention from existing works.  
%limits its adoption in resource-constrained nodes. 
In decentralized environment, %At an iteration, 
each learning node is required to compute the inverse of a Hessian matrix for the lower-level optimization per training round, 
incurring a complexity order of $\mathcal{O}(d^3)$, where $d$ is the model size. 
Instead of sending gradients, it needs to share its Hessian matrix, which has an order of $\mathcal{O}(d^2)$ \cite{pmlr-v162-tarzanagh22a} for communication.

Above challenges for solving bilevel optimization are further exacerbated by privacy preservation constraints, which give rise to the emergence of decentralized federated learning (DFL) \cite{10.1145/3494834.3500240}. 
DFL is a serverless variant of federated learning (FL) for coordinating multiple nodes to co-train a global model. 
In DFL, learning nodes perform local model updates and share them directly with one another.
Despite its flexibility, efficiently solving bilevel optimization is especially critical for DFL because \cite{pmlr-v202-shi23d}: 1) the limited computational resources on learning nodes make solving complex bilevel problems challenging; and 2) the heterogeneous data distribution across nodes coupled with the lack of a central server slows down convergence, leading to increased computational and communication overhead.

This paper addresses the problem of 
$m$ clients collaboratively solving a hyper-objective bilevel problem in a decentralized manner \cite{pmlr-v202-chen23n, NEURIPS2022_01db36a6}, as formulated in equation \eqref{opt_problem}:
\begin{equation}
\begin{aligned}
    \min_{x \in \mathbb{R}^{d_x}} \psi(x) &:= \frac{1}{m}\sum_{i=1}^m{f_i(x, y^\ast(x))},\\
    y^\ast(x) = \underset{y \in \mathbb{R}^{d_y}}{\arg\min}\, & g(x,y) := \left\{\frac{1}{m}\sum_{i=1}^m{ g_i(x, y)}\right\},
    \label{opt_problem}
\end{aligned}
\end{equation}
where $m$ is the number of nodes in the decentralized topology.
In this context, each learning node deals with distinct upper-level (UL) and lower-level (LL) problems, $f_i$ and $g_i$ respectively, derived from their various local datasets. 

To tackle the problem above in a resource-friendly manner, we proposes a \textbf{C}ompressed \textbf{C}ommunication-efficient \textbf{D}ecentralized \textbf{F}irst Order \textbf{B}ilvel Optimization method ($\text{C}^2$DFB)
by relying solely on first-order oracles and transmitting compressed residuals of local parameters.
Specifically, our algorithm solves a reorganized min-min-max problem instead of the original bilevel problem, enabling us to use only first-order oracles, which significantly reduces the computational complexity of acquiring the hypergradient compared to the state-of-the-art baselines. %bilevel decentralized algorithms. % We demonstrate that this approximation is a reliable proxy for the original problem, even in a decentralized setting.  
To further enhance communication efficiency, we devise the reference point of consensus values in $\text{C}^2$DFB, through which neighbor nodes can compress residuals of updated models and their reference points with mitigated error accumulation for both biased and unbiased compressors while avoiding extra transmission of error information.
%Additionally, each client maintains a reference point of consensus values, which are updated by its neighbors' compressed residuals of updated models and their reference points. This technique mitigates error accumulation for both biased and unbiased compressors while avoiding extra transmission of error information. 
Furthermore, $\text{C}^2$DFB is flexible in  incorporating strategies to boost compute efficiency,
such as gradient tracking and consensus step size adjustments. 
Theoretical analysis proves the convergence of our algorithm, %ensuring it can 
achieving the $\epsilon$-first order stationary point of the hyper-objective. 

In summary, our contributions are:
\begin{itemize}
    \item We propose % a Compressed Communication-efficient Decentralized First Order Bilevel Optimization method ($\text{C}^2$DFB) 
    the $\text{C}^2$DFB algorithm specifically tailored for decentralized settings. Our algorithm leverages first-order gradient oracles without the  need for second-order information, significantly alleviating computational overhead.
    \item We introduce a reference point-based compression strategy to diminish communication volume by transmitting only compressed residuals of local parameters. This approach ensures that both consensus and compression errors are mitigated during the training process. %, for both biased and unbiased compressors.
    \item We provide the theoretical guarantee of convergence for our $\text{C}^2$DFB algorithm %a decentralized fully first-order bilevel algorithm 
    without additional assumption on gradient or data heterogeneity. Our result demonstrates that it takes $\mathcal{O}(\epsilon^{-4}\log(\epsilon^{-4}))$ gradient calls to reach the first-order stationary point.
    \item We conduct extensive experiments over various topology, heterogeneity and models to evaluate $\text{C}^2$DFB. % the effectiveness of our method. 
    The results indicate that our approach remarkably surpasses existing second-order-based methods and single-loop methods in terms of both convergence rate and communication efficiency.
\end{itemize}
% \paragraph{Paper Organization.}

The remainder of this paper is organized as follows: Section \ref{sec:relatedwork}
reviews related works, highlighting key advancements, and the ongoing challenges in the field. Section \ref{sec:preliminary} introduces the preliminary concepts which are necessary for understanding our approach. 
In Section \ref{algori}, we detail the design of our Compressed Communication-efficient Decentralized First Order Bilevel Optimization (C$^{2}$DFB) algorithm. Section \ref{sec:convergence} provides a rigorous convergence analysis, demonstrating the effectiveness of C$^{2}$DFB. 
Finally, Section \ref{sec:exp} presents experimental evaluations conducted on two distinct tasks across various experimental settings.
\section{Related work}\label{sec:relatedwork}

\subsection{Bilevel problem optimization}

It is non-trivial to approximately derive the  stationary point in \eqref{opt_problem}, particularly when the UL function $\psi(x)$ is nonconvex and may exhibit non-differentiability or discontinuity. A widely accepted approach in the literature \cite{NEURIPS2022_c5cf13bf, NEURIPS2022_1413947e} assumes strong convexity in the LL problem with respect to $y$. Under this condition, %the hyper-gradient $\nabla \psi(x)$ can be expressed as
%\begin{equation}
%\begin{aligned}
 %   \nabla \psi(x) &= \left(\frac{1}{m}\sum_{i=1}^m \nabla_x f_i(x, y^*(x))\right) - \nabla_{xy} g(x, y^*(x)) \\ &\quad \left[ \nabla_y^2 g(x, y^*(x)) \right]^{-1} \left(\frac{1}{m}\sum_{i=1}^m \nabla_y f_i(x, y^*(x))\right).
%\end{aligned}
%\label{hyper_grad}
%\end{equation}
%The formulation necessitates the use of
it is necessary to use the second-order methods due to the involvement of the Hessian matrix.

However, the process of inverting this matrix introduces nonlinearity computation, making bilevel optimization challenging in decentralized environment because %as illustrated below
%\begin{align*}
%     \nabla_{xy} g(x, y^*(x)) & \left[ \nabla_y^2 g(x, y^*(x)) \right]^{-1} \neq
 %    \\&\frac{1}{m}\sum_{i=1}^m \nabla_{xy} g_i(x, y^*(x)) \left[ \nabla_y^2 g_i(x, y^*(x)) \right]^{-1}.
%\end{align*}
%This disparity indicates that 
averaging the local hyper-gradients does not exactly equate to the global gradients. Consequently, there is an imperative need to compute and communicate the Hessian matrix of local objectives. Performing such computations often becomes impractical for resource-limited nodes commonly found in decentralized systems \cite{pmlr-v202-chen23n}.

\subsection{Decentralized bilevel optimization} 
% Federated bilevel optimization has been studied in the literature. 

Recently, bilevel optimization in decentralized systems receives arising attention due to the emergence of FL,
%To scale the method efficiently across clients, 
%approaches have been investigated in designing algorithms specifically for
invoking exploration on federated bilevel optimization problems \cite{pmlr-v162-tarzanagh22a, NEURIPS2023_04bd683d,NEURIPS2023_686a3f32} .

FedNest \cite{pmlr-v162-tarzanagh22a} incorporated variance reduction and inverse Hessian-gradient product approximation techniques  to improve the applicability of federated algorithms in bilevel optimization. FedBiOAcc \cite{NEURIPS2023_04bd683d}  was proposed to efficiently evaluate the hyper-gradient in a distributed setting by formulating it as a quadratic federated optimization problem, which can be accelerated by utilizing momentum-based variance reduction techniques. % to accelerate convergence. 
SimFBO \cite{NEURIPS2023_686a3f32} further addressed challenges posed by extensive matrix-vector product calculations in federated settings. 
This is achieved through updating the UL and LL models in a single loop iteration, aligning with an auxiliary vector for Hessian-inverse vector products via the gradient of a local quadratic function.
% % FedBiOAcc \cite{NEURIPS2023_04bd683d} formulates a federated quadratic optimization problem and utilizes momentum-based variance reduction techniques to accelerate its convergence. 

% % XU: SimFBO\cite {NEURIPS2023_686a3f32} addresses what challenges? and how do they achieved that? (revised)

% \subsection{Decentralized bilevel optimization methods}
% XU: Never using ``several works, several studies, several papers'' (revised)

% XU: First of all, explain the feature of decentralized bilevel optimization and descirbe the Hessian gradient vector cannot be aggregated at once (consensus cannot be reached at once) .(unclear, previous works do not focus on the consensus, they mainly focus on solving the problem of eq(2). I could state this challenge here but  formulas will also be engaged)

DFL, as a serverless variant of FL, has attracted  research efforts to %es also designed various bilevel algorithms to address the above mentioned 
 improve bilevel optimization computational and communication efficiency. \citet{pmlr-v202-chen23n} proposed to use the Hessian-Inverse-Gradient-Product (HIGP) oracle by leveraging a quadratic sub-solver to approximate the product of the inverse Hessian and gradient vectors without computing the full Hessian or Jacobian matrices. % \cite{pmlr-v202-chen23n}. 
 \citet{NEURIPS2022_01db36a6}  introduced a method to estimate this product in a decentralized manner. Specifically, it approximates the Hessian-inverse by recursively computing the Neumman series. However, in some applications \cite{pmlr-v70-finn17a,nichol2018firstordermetalearningalgorithms}, computing Jacobian/Hessian-inverse vector products heavily consumes more computational and memory resources than  obtaining the gradient. This becomes particularly prohibitive in resource-limited decentralized settings.

% XU: What do you mean collaborative estimation? (revised)

% XU: It seems that we might introduce the works using the more expensive Hessian vector first, and then those using gradient vector.(unclear, no decentralized bilevel using gradients oracle only for now)

It is worth mentioning that additional efforts have been dedicated to %, other investigations attempt to 
mitigating communication overhead by averting the inner-loop communications through single-loop algorithm design \cite{dong2024singleloopalgorithmdecentralizedbilevel, pmlr-v206-gao23a}. \citet{dong2024singleloopalgorithmdecentralizedbilevel} proposed a single-loop method with an asymptotic rate of $\mathcal{O}(1/\sqrt{T})$. \citet{pmlr-v206-gao23a} combined local Neumann-type approximations and gradient tracking to design a single-loop method, though it explicitly required the homogeneous data distribution across all devices. Besides, there is a lack of proof for non-asymptotic stage convergence for loopless algorithms, and the reliability remains unknown. %to be investigated.

% XU: Never using ``some''. It seems that the term ``single loop'' suddenly  jumps out.(single loop means no inner loops)

% XU: Why is their reliability unclear?(it means that the theoretical investigation is still insufficient)
Building on existing research, we enhance the computational and communication efficiency of bilevel optimization by computing and sharing only compressed first-order gradients among resource-constrained learning nodes in decentralized systems.

%To the best of our knowledge, decentralized bilevel optimization has yet to be addressed using only first-order gradients, without relying on the computationally intensive Hessian-inverse vector. Furthermore, the application of communication compression techniques to first-order federated bilevel optimization remains unexplored.

% XU: The above statement is rather weak, causing the main difference of your work very vague. (revised)

\section{Preliminary}
\label{sec:preliminary}

% XU: The first sentence needs to mention the decompositin of a bilevel optimization problem into an upper-level and a lower-level subproblem? (mentioned before in the intro)

% XU: Still unclear why the inverse of the matrix cause significant communication overhead. The explanation is indirect. Besides, the fully first order method is not proposed for communication purposes. (revised, adding one formula)

\subsection{Fully-first order bilevel optimization}
The objective of %optimizing the single-machine 
centralized bilevel optimization, as outlined in \eqref{opt_problem}, is to find an $\epsilon$-stationary point of $\psi(x)$.

\begin{definition}
    A point $x$ is called an $\epsilon$-stationary point of a differentiable function $\psi(x)$ if $\|\nabla\psi(x)\| \leq \epsilon.$
\end{definition}

Typically, solving this problem queries the second-order information of the UL%upper-level (UL) 
function $g$ to obtain the hyper-gradient $\nabla \psi(x)$, % via \eqref{hyper_grad}. However, since calls to second-order oracles can be 
which however is computationally expensive. A practical approach is to approximate the hyper-gradient by only using gradients. %information. 
% XU: What do you mean the ``calls to such oracles''? (calls to second order oracles)
To achieve this, \citet{pmlr-v202-kwon23c} reformulates \eqref{opt_problem} as a constrained optimization problem:
\begin{equation}
    \begin{aligned}
        \min_{x \in \mathbb{R}^d_x} \psi_\lambda^\ast(x)
        :&= \min_{y \in \mathbb{R}^{d_y}} f(x,y) + \lambda \bigg( g(x,y) - g^\ast(x) \bigg) , \\
         g^\ast(x) &= \min_{z \in \mathbb{R}^{d_y}} g(x,z).
    \end{aligned}
    \label{reform}
\end{equation}

In this formulation, the Lagrangian includes a multiplier $\lambda$, where $\lambda > 0$. %It is proved to be 
Here, $\nabla \psi^\ast_\lambda(x) $ approximates $\nabla \psi(x)$ with
\begin{equation}
    \left\| \nabla \psi^\ast_\lambda(x) - \nabla \varphi(x) \right\| = \mathcal{O} \left( \frac{\kappa^3}{\lambda} \right).
\end{equation}
Choosing $\lambda$ on the order of $\mathcal{O}(\epsilon^{-1})$ ensures that  the $\epsilon$-stationary point of $\psi_\lambda(x)$ is also an $\epsilon$-stationary point of $\psi(x)$. Denote $y^\ast_{\lambda}(x)$ as
\begin{align*}
    y_{\lambda}^*(x)= \arg \min_{y \in \mathbb{R}^{d_y}} f(x,y) + \lambda \bigg( g(x,y) - g^\ast(x) \bigg).
\end{align*}
Then, the closed form of $\nabla \psi_\lambda(x)$ can be derived using only the first-order gradients of $f$ and $g$ by:
\begin{equation}
    \begin{aligned}
&\nabla \psi^\ast_{\lambda}(x)\\ &= \nabla_x \psi_{\lambda}(x, y_{\lambda}^*(x)) + \nabla_y y_{\lambda}^*(x) \nabla_y \psi_{\lambda}(x, y_{\lambda}^*(x)) \\
&= \nabla_x f(x, y_{\lambda}^*(x)) + \lambda \bigg(\nabla_x g(x, y_{\lambda}^*(x)) - \nabla_x g(x, y^*(x))\bigg).
\end{aligned}
\label{reform_grad}
\end{equation}
%Reformulating the problem in \eqref{reform_grad} makes bilevel optimization feasible for resource-limited decentralized settings.

% XU: It seems that all these sentences are in the form ``this ***''. Please rewrite these sentences. (revised)

\subsection{Contractive compressor}
%To design communication-efficient protocols,
Model compression has been widely used to boost communication efficiency in decentralized learning such as sparsification and quantization, which can be customized to combine with our bilevel optimization solution.  % are widely utilized in decentralized or federated learning. 
A general compression operator, known as the contractive compressor, is defined as below \cite{NEURIPS2018_44feb009}.
\begin{definition}
The contractive compression operator $\mathcal{Q} \colon \mathbb{R}^{m \times d} \rightarrow \mathbb{R}^{m \times d}$ satisfies
    \begin{align*}
        \mathbb{E} \left[ \|\mathcal{Q}(A) - A\|^2 \right] \leq (1 - \delta_c)\|A\|^2
    \end{align*}
    for all $A \in \mathbb{R}^{m \times d}$, where $\delta_c \in (0, 1]$ is a constant determined by the compression operator $\mathcal{Q}$.
\label{def_contractive_compressor}
\end{definition}

It is important to note that the contractive compressor encompasses a wide range of both unbiased and biased compressors. For an unbiased compressor, Definition \ref{def_contractive_compressor} can be extended with the additional condition $\mathbb{E} \left[ \mathcal{Q}(A) \right]=A$. For a biased operator, we should replace $Q$ and $\delta_c $ with  $\mathcal{Q}^\prime = \frac{\mathcal{Q}}{2-\delta_c}$  and $\delta'_c = \frac{1}{2-\delta_c}$ (See proof in the supplementary material), respectively, in Definition \ref{def_contractive_compressor}.
%then the new operator also satisfies Definition \ref{def_contractive_compressor} with 

% XU: I can understand what you are saying. But it takes some efforts to understand how definition 2 can be generalized to the unbiased compression. (yes, i admit. but it is hard to be more intuitive)

\section{Decentralized Billevel Optimization}\label{algori}

\subsection{Parameter notations}
To  facilitate our  discussion, we introduce the following  notations. % for parameters.
% XU: Rewrite the above sentence. 
The parameters on  node $i$ are denoted by $(x_i)^t, (y_i^k)^t$, and $(z_i^k)^t$,  for the $k$-th iteration of the inner loop and the $t$-th iteration of the outer loop. 

We denote the average of the parameters as follows:
\begin{align*}
    \bar{x} = \frac{1}{m} \sum_{i=1}^m x_i, \quad
    \bar{y} = \frac{1}{m} \sum_{i=1}^m y_i, \quad
    \bar{z} = \frac{1}{m} \sum_{i=1}^m z_i.
\end{align*}

A stacked version of the global parameters is defined as:
\begin{align*}
    \mathbf{x} &= \begin{bmatrix}
        x_1,
        x_2,
        \cdots 
        x_m
    \end{bmatrix}^\top \in \mathbb{R}^{m \times d_x}.
    % \mathbf{y} &= \begin{bmatrix}
    %     y_1,
    %     y_2,
    %     \cdots 
    %     y_m
    % \end{bmatrix}^\top \in \mathbb{R}^{m \times d_y}\\
    % \mathbf{z} &= \begin{bmatrix}
    %     z_1,
    %     z_2,
    %     \cdots 
    %     z_m
    % \end{bmatrix}^\top \in \mathbb{R}^{m \times d_y},
\end{align*}

Additionally, let $\mathbf{1} = [1, \cdots, 1]^\top \in \mathbb{R}^m$, and $\|\cdot\|$ denote the 2-norm for vectors and the Frobenius norm for a matrix.

% XU: I cannot understand the above sentence.(revised)

\subsection{System modeling}

% XU: try to avoid using ``in this paper'' in the main body of your paper.(revised)

% XU: Better to add a figure to explain how a node interacts with its neighbours in terms of data transfer. (Revised, see fig.1)

We follow the method to reformulate the bilevel problem in \eqref{reform_grad} to make it feasible for resource-limited decentralized settings.
%The optimization of problem \eqref{opt_problem} over 
Consider a decentralized graph $\mathcal{G}$ consisting of a set of nodes $\mathcal{V} = \{1, 2, \ldots, m\}$ and an edge set $\mathcal{E}$. The neighbors of node $i$ are denoted by the set $\mathcal{N}_i$.
We define the mixing matrix $\mathbf{W}$ as follows and impose common assumptions on the decentralized problem:

\begin{assumption}
    To align with networks in real-world scnearios, the graph $\mathcal{G} = (\mathcal{V}, \mathcal{E})$ is connected and undirected, which can be represented by a mixing matrix $\mathbf{W} \in \mathbb{R}^{m \times m}$. Let $w_{ij}$ be the element in the $i-$th row and $j-$th column of matrix $\mathbf{W}$, the following properties hold:
    \begin{itemize}
        \item[1)] $w_{ij} > 0$ if $(i,j) \in E$, and $w_{ij} = 0$ otherwise.
        \item[2)] $\mathbf{W}$ is doubly stochastic, \emph{i.e.}, $\mathbf{W} = \mathbf{W}^\top$, $\sum_{i=1}^m w_{ij} = 1$, and $\sum_{j=1}^m w_{ij} = 1$.
        \item[3)] The eigenvalues of $\mathbf{W}$ satisfy $\lambda_m \leq \ldots \leq \lambda_2 \leq \lambda_1 = 1$ and $\nu = \max\{|\lambda_2|, |\lambda_m|\} < 1$.
    \end{itemize}
    \label{assmp_graph}
\end{assumption}
Based on $\mathbf{W}$, we define the spectral gap, a well-known measure indicating how well  nodes are connected. 

% XU: In fact, the spectral gap measures how ``dense'' a graph is or how well the nodes are connected. (revised)

\begin{definition}
    For a gossip mixing matrix $\mathbf{W}$ that satisfies Assumption 1, we define the spectral gap as $\rho \equiv 1 - \delta_\rho$, where $\delta_\rho \equiv \max\{|\lambda_2(\mathbf{W})|, |\lambda_m(\mathbf{W})|\}$ is the second largest eigenvalue in the magnitude.
\end{definition}

% XU: never using ``some'' since this world is rather fuzzy. 

We also make the following standard assumptions regarding the smoothness of the UL and LL problems~\cite{chen2023nearoptimalnonconvexstronglyconvexbileveloptimization}.

\begin{assumption}
    Recall the definitions of $f_i$ and $g_i$ in Eq \eqref{opt_problem}. We assume the following:
    \begin{enumerate}
        \item The upper-level function $f_i(x, y)$ is $C_f$-Lipschitz continuous in $y$, has $L_f$-Lipschitz continuous gradients, and $\rho_f$-Lipschitz continuous Hessians.
        \item The lower-level function $g_i(x, y)$ is $\mu$-strongly convex in $y$, has $L_g$-Lipschitz continuous gradients, and $\rho_g$-Lipschitz continuous Hessians.
    \end{enumerate}
    \label{assump_smooth}
\end{assumption}
\begin{definition}
    With Assumption \ref{assump_smooth}, we define $l = \max\{C_f, L_f, L_g, \rho_g\}$, and condition number as $\kappa = l/\mu$. 
\end{definition}
% add L_psi lipstchiz contiuous here as a preposition

\textit{\textbf{Remark}: Note  that our algorithm allows the heterogeneity between functions and the dissimilarity between local and global gradients is not bounded.}

% XU: unclear remark.(this remark means that we make a looser assumption than other works)

Based on Assumptions \ref{assmp_graph} and \ref{assump_smooth}, the reformulated problem in \eqref{reform} can be decomposed into two-level objectives.
% XU: describe the upper and the lower problems again without math, and then tell the objective of each of them.
For the outer level problem, our objective is to estimate $\nabla \psi_\lambda(\bar{x})$, similar to \eqref{reform_grad}. This is challenging because we can only compute local gradients $\nabla_i \psi_\lambda(x)$. However, previous studies have demonstrated that by employing consensus strategies such as gradient tracking, we can achieve the desired $\epsilon$-stationary point of $\nabla \psi_\lambda(\bar{x})$ at a high convergence speed.
\begin{equation}
    \begin{aligned}
        y^\ast &= \arg \min_{z \in \mathbb{R}^{d_y}} \frac{1}{m}\sum_{i=1}^m g_i(x,z), \\
        y_\lambda^\ast &= \arg \min_{y \in \mathbb{R}^{d_y}} \frac{1}{m}\sum_{i=1}^m h_i(x,y) := \left( f_i(x,y) + \lambda g_i(x,y) \right).
    \end{aligned}
    \label{approximation}
\end{equation}

Challenges in achieving these optimal solutions become pronounced in a decentralized systems. Notably, since each node only accesses its local UL %upper-level (UL) 
model $x_i$ rather than the global model $\bar{x}$, the system can only approximate optimal solutions, denoted as $\tilde{y}^\ast$ and $\tilde{y}_\lambda^\ast$.

% In our theoretical analysis, we demonstrate that this deviation introduces a consensus error in the UL model, which can be gradually reduced through iterations.

% XU: Do we need ``In our theoretical analysis, we demonstrate that''?

\begin{algorithm}[t]
\caption{$\text{C}^2$DFB \textbf{C}ompressed \textbf{C}ommunication-efficient \textbf{D}ecentralized \textbf{F}irst Order \textbf{B}ilvel Optimization}
\label{alg:C2DFB2}
  \SetKwData{Left}{left}\SetKwData{This}{this}\SetKwData{Up}{up}
  \SetKwFunction{Union}{Union}\SetKwFunction{FindCompress}{FindCompress}
  \SetKwInOut{Input}{input}\SetKwInOut{Output}{output}\SetKwInOut{Init}{initialize}
  
\Input{Initial local models $x_i^{0}$, $y_i^{0}$; Step size $\eta_{in}$, $\eta_{out}$; Mixing step size $\gamma_{in}$, $\gamma_{out}$; Inner-loop times $K$; Total iterations $T$; Penalty coefficient $\lambda$}
\Init{$z_i^0=y_i^{0}$, $(s_i)^0_x = (u_i)^0_x = \nabla_x f_i(x_i^0, y_i^0) + \lambda \left( \nabla_x g_i(x_i^0, y_i^0) - \nabla_x g_i(x_i^0, z_i^0) \right)$}
\BlankLine
\For{$t=
0,\dots T-1$}{
 \For{$i=1$ to $m$ in parallel}{
    {\footnotesize\textcolor{blue}{Outer Loop Update, communicate $x$}:}\\
    $x_i^{t+1} = x_i^{t} + \gamma\sum_{j \in \mathcal{N}_i} w_{i,j}\{x_j^{t}-x_i^{t}\} -\eta(s_i)^{t}_x$\\

    {\footnotesize\textcolor{blue}{Inner Loop Update}}:\\
    $y_i^{t+1} = {\footnotesize\text{\textbf{IN}}}\big(h (x_i^{t+1},y), y_i^{t}, (\hat{y}_i^K)^t, (s_i^K)_y^t, (\hat{s}_i^K)_y^t\big)$\\
    $z_i^{t+1} = {\footnotesize\text{\textbf{IN}}}\big(g(x_i^{t+1},y), z_i^{t}, (\hat{z}_i^K)^t,   (s_i^K)_z^{t}, (\hat{s}_i^K)_z^t\big)$\\

     {\footnotesize\textcolor{blue}{Local Gradients Computation}}:\\
    $(u_i)^{t+1}_x = \nabla_x f_i(x_i^{t+1}, y_i^{t+1}) + \lambda\left( \nabla_x g_i(x_i^{t+1}, y_i^{t+1}) - \nabla_x g_i(x_i^{t+1}, z_i^{t+1}) \right)$\\

      {\footnotesize \textcolor{blue}{Gradient tracker update, communicate $s_x$}}:\\
     $(s_i)^{t+1}_x = (s_i)^{t}_x + \gamma \sum_{j \in \mathcal{N}_i} w_{i,j}\{(s_j)^{t}_x-(s_i)^{t}_x\} + (u_i)^{t+1}_x - (u_i)^{t}_x$\\
    }
}
\end{algorithm}

\begin{algorithm}[t]
    \caption{Inner loop update on the $i$-th client \\ \quad \textbf{IN} \big($r_i(d)$, $d_i^0$, $\hat{d}_i^0$, $s_i^0$, $\hat{s}_i^0$, $\gamma$, $\eta$, $K$\big)}
    \label{alg:ref_compress_beer}
    \SetKwData{Left}{left}\SetKwData{This}{this}\SetKwData{Up}{up}
    \SetKwFunction{Union}{Union}\SetKwFunction{FindCompress}{FindCompress}
    \SetKwInOut{Input}{Input}\SetKwInOut{Output}{Output}
    
    \Input{Initial parameters $d_i^0$, $s_i^0$; Reference points $\hat{s}_i^0$, $\hat{d}_i^0$; Objective function $r_i(d)$; Step size $\eta$; Mixing step size $\gamma$; Iteration $K$}
    
    \For{$k = 0, 1, \ldots K-1$}{
            $d_i^{k+1} = d_i^k + \gamma \sum_{j \in \mathcal{N}_i} w_{i,j} \{\hat{d}_j^k - \hat{d}_i^k\} - \eta s_i^k$ \\
            {\footnotesize\textcolor{blue}{communicate $\mathcal{Q}(d_i^{k+1} - \hat{d}_i^{k})$}:}\\
            $\hat{d}_i^{k+1} = \hat{d}_i^{k} + \mathcal{Q}(d_i^{k+1} - \hat{d}_i^{k})$\\ 
            $s_i^{k+1} = s_i^k + \gamma \sum_{j \in \mathcal{N}_i} w_{i,j} \{\hat{s}_j^k - \hat{s}_i^k\} + \nabla r_i^{k+1} - \nabla r_i^{k}$\\
            {\footnotesize\textcolor{blue}{communicate $\mathcal{Q}(s_i^{k+1} - \hat{s}_i^{k})$}:}\\
            $\hat{s}_i^{k+1} = \hat{s}_i^{k} + \mathcal{Q}(s_i^{k+1} - \hat{s}_i^{k})$
    }
    \Output{$d_i^K$}
\end{algorithm}

Furthermore, there is a trade-off between the final solution  precision  and communication cost. On the one hand, executing more iterations within the inner loop makes the solution closer to the optimal. On the other hand, enhancing communication across the network reduces disparities between local and global parameters. However, both of these approaches incurs more communication overhead.

% Second, there is a trade-off between the precision of the final solution and the communication cost. This trade-off manifests in two aspects. On one hand, more inner loop iterations yield a result closer to the optimal solution. On the other hand, increased communication across the network reduces the discrepancy between local and global parameters. However, both of these approaches result in higher communication overhead.
% XU: It seems that you asks the right question, but do not solve it. 

To address these challenges, we  specially design compression with reference point to mitigate the communication efficiency concern. This technique largely decreases the inner-loop communication volume while still guarantees the same convergence performance of the inner loop. Additionally, we utilize gradient tracking and mixing step to accelerate global consensus across inner and outer loops. 

% XU: The global consensus error stems from the local aggregation. The above paragraph is too short. It does not clearly deliver your idea and highlight how novel or how important they are. 

\subsection{Algorithm design}
% three strategy, compression, gradient tracking, mixing parameter $gamma$.

% Shall we place a model section before algorithm design? 

We embark to present the $\text{C}^2$DFB (\textbf{C}ompressed \textbf{C}ommunication-efficient \textbf{D}ecentralized \textbf{F}irst Order \textbf{B}ilvel Optimization) method. As described in Algorithm \ref{alg:C2DFB2} - \ref{alg:ref_compress_beer}, the algorithm operates in two main loops. 

In each iteration of the Outer loop, the UL model $x_i$ is updated using the gradient tracker $(s_i)^{t}_x$. This update incorporates a weighted averaging, also known as the mixing step, which blends parameters from neighboring nodes with the node's local parameters. Subsequent to this, the inner loop proceeds to update the LL model $y_i$ and $z_i$. Following  this, the estimation of UL gradients $(u_i)^{t+1}_x$ is computed to update the tracker $(s_i)^{t+1}_x$. For a better clarity, default hyperparameters such as  $\gamma_{in}$ and $\eta_{in}$ are omitted when invoking the \textbf{IN} function. 

\begin{figure}[h]
  \centering
  \includegraphics[width=\linewidth, height=3cm]{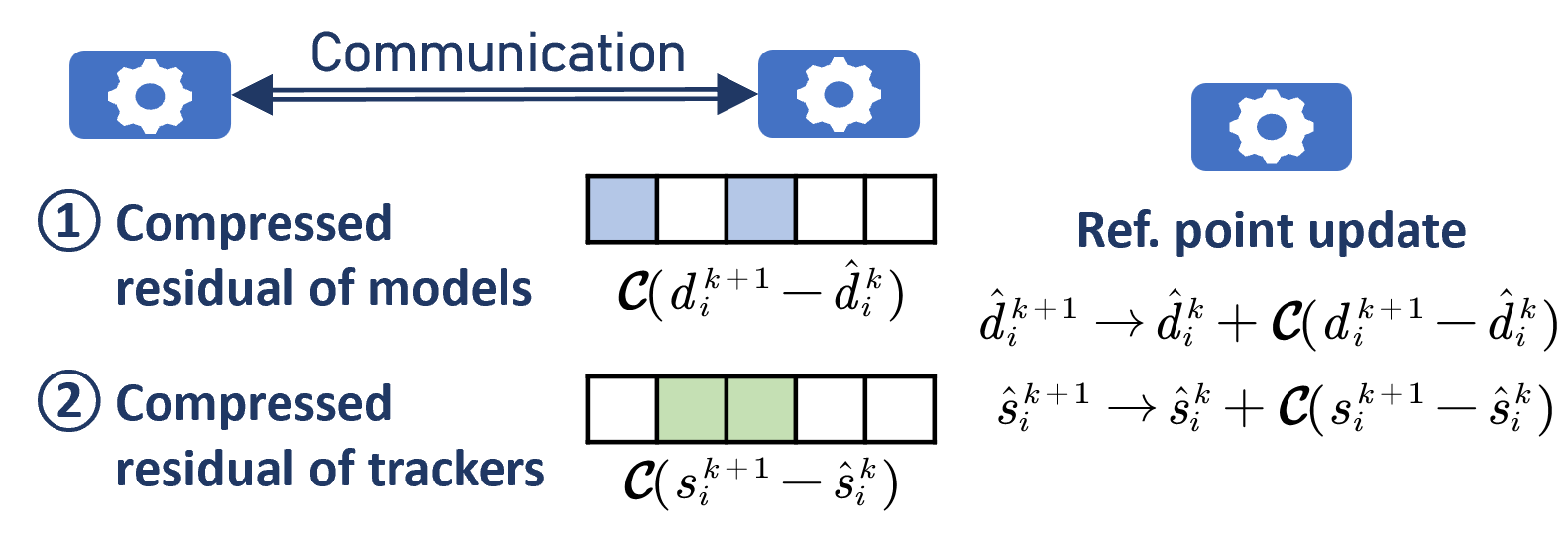}
  \caption{Communication protocol of inner loop for $\text{C}^2$DFB}
  \label{fig:comm_iloop}
\end{figure}
Within the inner loop, a three-step optimization is conducted below. First, each training step begins with updating  model $d_i$(which may represent either $y_i$ or $z_i$) using the gradient tracker. A consensus mechanism refines the update by computing a weighted average of the models from neighboring nodes, controlled by the mixing step size  $\gamma$.  $\gamma$ increases the adaptability of the system to reach consensus while providing stability against aggressive compression strategies, particularly under the heterogeneity of local objectives.

In the second step, a local reference point $\hat{d}_i$ is updated by incorporating a compressed residual. This residual compress the deviation between the current model parameter $d_i$ and the reference point $\hat{d}_i$. As the training advances, $\hat{d}_i$ increasingly aligns with $d_i$, evidenced by the diminishing magnitude of $d_i^{k+1} - \hat{d}_i^k$ \cite{9789732}. Unlike traditional compression methods that merely transmit compressed parameters and accumulate the compression error for subsequent updates, this approach distinguishes itself in two aspects:

1) Implicit error compensation is facilitated through this technique. To illustrate, let the error be defined as $e_i^k = \hat{d}_i^k - d_i^k = \mathcal{Q}(d_i^k - \hat{d}_i^{k-1}) - (d_i^k + \hat{d}_i^{k-1})$. We can reformulate the model update as:
\begin{equation}
    \begin{aligned}
        d_i^{k+1} &= d_i^k + \gamma \sum_{j \in \mathcal{N}_i} w_{i,j} \{d_j^k+e_j^k - d_i^k - e_i^k\} - \eta s_i^k \\
    = & (1-\gamma) d_i^k +  \gamma\sum_{j \in \mathcal{N}_i} w_{i,j} d_j^k + \gamma\sum_{j \in \mathcal{N}_i, j \neq i}w_{i,j} e_j^k - \eta s_i^k.
    \end{aligned}
    \label{err_compensate}
\end{equation}
Here, each iteration inherently transmits the compression error to neighboring clients.

2) This method ensures that the update step remains consistent with a global update strategy that does not involve compression. Once received, the local parameter integrates this error and compresses it for the subsequent iteration. Globally, averaging over $m$ clients, the following result holds based on the fact that $\textbf{1}^\top(\mathbf{W}-\mathbf{I})=0$:
\begin{equation}
\begin{aligned}
    \bar{d}^{k+1} &= \bar{d}^k + \gamma \textbf{1}^\top(\mathbf{W}-\mathbf{I}) (\hat{d}^k) - \eta \bar{s}^k \\
    &= \bar{d}^k-\eta \bar{s}^k.
\end{aligned}
\label{compress_prop2}
\end{equation}
Practically, each client transmits $\mathcal{Q}(d_i^{k+1} - \hat{d}_i^{k})$, while maintaining an additional variable to represent the accumulated reference point from its neighbors, denoted by $(\hat{d}_i^{k+1})_w = (\hat{d}_i^{k})_w + \sum_{j \in \mathcal{N}i} w_{i,j} \mathcal{Q}(d_i^{k+1} - \hat{d}_i^{k})$. This confirms that $(\hat{d}_i^{k+1})_w = \sum_{j \in \mathcal{N}_i} w_{i,j} \hat{d}_i^{k}$.

Lastly, a similar consensus mixing and reference point update mechanism is applied to the gradient tracker. The global gradient is tracked by accumulating the residuals from two consecutive iterations. This method, commonly known as gradient tracking \cite{NEURIPS2021_5f25fbe1}, is prevalent in decentralized systems and ensures that the tracker $\bar{s}^k$ reliably represents the average gradient across $m$ clients. Notably, this approach enables gradient descent to achieve linear convergence without necessitating uniformity in local gradients.

In conclusion, $\text{C}^2$DFB enhances both computational and communication efficiency  through three key mechanisms. Due to its fully first-order nature, it simply queries gradient oracles that are easy to compute and transfer. Within the inner loops, it efficiently maintains aggregated reference points of neighboring clients through the minimal transmission of compressed residuals. Simultaneously, robust performance is guaranteed by employing gradient tracking alongside a mixing consensus step, which ensures the convergence of the optimization process in decentralized systems.

\section{Convergence analysis}\label{sec:convergence}
In this section, we present a value-function-based framework to analyze the convergence of the $\text{C}^2$DFB algorithm. We define recursive error bounds for the outer and inner loops, then construct a value function from a linear combination of these errors. Our analysis shows that the algorithm requires $\mathcal{O}(\epsilon^{-4})$ outer-loop iterations and $\mathcal{O}(\log{1/\epsilon^{4}})$ inner-loop first-order oracle calls to achieve convergence.

% XU: please simplify the above paragraph. Using three sentences for an abstract description. 

\subsection{Value functions framework}
Our convergence analysis is grounded on Lyapunov functions. Given that bilevel problem  comprises two loops, we delibrately design Lyapunov functions for each loop.
For the outer-loop,  we design the following Lyapunov functions: 
\begin{equation}
    \begin{aligned}
    \text{Model consensus error:   } \Omega^t_1& := \|\mathbf{x}^t - \textbf{1}\bar{x}^t\|^2 \\
    \text{Tracker consensus error:   } \Omega^t_2& := \|\mathbf{s}_x^t - \textbf{1} \bar{s}_x^t \|^2\\
    \text{Value function: }
    \Omega^t \triangleq \psi(\bar{x}^t) &+ \frac{1}{m}\Omega^t_1 + \frac{\eta_{out}}{m}\Omega^t_2,
    \end{aligned}
    \label{valuefunc_outerloop}
\end{equation}

As for the inner-loop, our Lyapunov function are delibrately designed as:
\begin{equation}
    \begin{aligned}
        \text{Model compression error:   } \Omega^k_1 & := \|\mathbf{d}^k - \hat{\mathbf{d}}^k \|^2 \\
        \text{Model consensus error:   } \Omega^k_2& := \|\mathbf{d}^k - \textbf{1}\bar{d}^k\|^2 \\
        \text{Model compression error: } 
        \Omega^k_3& :=  \|\mathbf{s}_d^k - \hat{\mathbf{s}}_d^k \|^2\\
        \text{Tracker consensus error:   } \Omega^k_4 & := \|\mathbf{s}^k - \textbf{1} \bar{s}^k \|^2\\
        \Omega^k \triangleq r(\bar{d}^k) + \frac{L_r}{m} \Omega^k_1 + \frac{1}{m L_r}\Omega^k_2 &+ \frac{L_r}{m}\Omega^k_3 + \frac{1}{m L_r}\Omega^k_4,\\
    \end{aligned}
    \label{valuefunc_innerloop}
\end{equation}
To make our presentation concise,  $\mathbf{d}$ can denote $\mathbf{y}$ or $\mathbf{z}$. $r$ may represents the objective function $h$ and $g$. $L_r$ is the corresponding Lipschitz parameter $2\lambda L_g$ and $L_g$ respectively.

The core idea behind these designed functions is that the errors recursively intertwine with each other. Through a value function, we can unify these complex relationships by progressively decreasing the value function, such that the algorithm converges to a stable state within a finite number of iterations. We will briefly outline the proof with key lemmas and highlight the final result.

% inner_loop errors, compression, model consensus, tracker consensus, innerloop value function 
% outer_loop model consensus, tracker consensus, outer loop value function

\subsection{Convergence Results}
We first provide the following lemma which shows that $\psi_\lambda(x)$ is an effective proxy of $\psi(x)$.
\begin{lemma}[\cite{pmlr-v202-kwon23c,chen2023nearoptimalnonconvexstronglyconvexbileveloptimization}]
    \label{first-order}
    Under Assumption \ref{assump_smooth}, if $\lambda \geq 2L_f/\mu$, it holds that 
    \begin{itemize}
        \item[1)] $\|\nabla \psi_\lambda(x)-\nabla \psi(x)\| \leq \mathcal{O}(\frac{l\kappa^3}{\lambda})$ and $\| \psi_\lambda(x)- \psi(x)\| \leq \mathcal{O}(\frac{l\kappa^2}{\lambda})$
        \item[2)] $\psi_\lambda(x)$ has a $\mathcal{O}(\kappa^3)$ gradient.
        \item[3)] $h_i(x,y)$ defined in \eqref{approximation} is $\frac{\lambda \mu}{2}$ strongly-convex in $y$.
    \end{itemize}
\end{lemma}

Before presenting the convergence of inner-loop optimization, a key problem lies in, $y^\ast(\bar{x})$ and $y^\ast_\lambda(\bar{x})$ is the optimum based on a consensus $x$. 
However, we could only access the local upper model $x_i$ to achieve an approximate optimum, $\tilde{y}^\ast$ and $\tilde{y}^\ast_\lambda$. 
Such deviation is further bounded by the consensus error of upper model in the following lemma, which performs as the bridge between the double loop.
\begin{lemma}
    \label{lipstchiz}
    Under Assumptions \ref{assmp_graph} and \ref{assump_smooth}, the following statements hold,
    \begin{align*}
        \|\tilde{y}_\lambda^\ast-y_\lambda^\ast(\bar{x}^t)\|^2 &\leq \frac{16\kappa^2}{m}\|\mathbf{x}^t - \textbf{1}\bar{x}^t\|^2 \\
         \|\tilde{y}^\ast-y^\ast(\bar{x}^t)\|^2 &\leq \frac{\kappa^2}{m}\|\mathbf{x}^t - \textbf{1}\bar{x}^t\|^2
    \end{align*}
\end{lemma}

Then, we offer a linear convergence guarantee for inner loop approximation.
\begin{theorem}
    \label{inner_loop_convergence}
    If Assumptions \ref{assmp_graph} and \ref{assump_smooth} hold, for the inner loop of K steps, at iteration $t$, there exist $\gamma_{in} \propto \delta_c \rho, \eta_{in} \propto \frac{\delta_c \rho^2}{\kappa \lambda L_g}$ such that
    \begin{align*}
        &\|(\mathbf{y}^K)^t - \mathbf{1}\tilde{y}_\lambda^\ast\|^2 \leq C_y \exp(-\frac{\mu K}{8L_g})~~~\text{and}\\
        &\|(\mathbf{z}^K)^t - \mathbf{1}\tilde{y}^\ast\|^2 \leq C_z \exp(-\frac{\mu K}{2L_g}),
    \end{align*}
    where $C_y$ and $C_z$ are positive constants. 
\end{theorem}
%emark: Theorem 1 shows a linear convergence rate of
%inner loop, which is aligned with the result for a strongly-
%convex decentralized algorithm along with gradient track-
%ing.(cite)
\begin{corollary}
For any positive constant $\epsilon \ge 0$, $ \|(\mathbf{y}^K)^t - \mathbf{1}\tilde{y}_\lambda^\ast\|^2$ and $\|(\mathbf{z}^K)^t - \mathbf{1}\tilde{y}^\ast\|^2$ can reach an order of $\mathcal{O}(\epsilon^{-2})$ if inner loop steps $K = \mathcal{O}(\log{(1/\epsilon^{4})})$.
\end{corollary}

\begin{lemma}
    \label{outer_loop_error}
    Based on the definition \eqref{valuefunc_outerloop}, the following error could be recursively bounded by,
    \begin{align*}
       \Omega^{t+1}_1 & \leq (1-\frac{\gamma\rho}{2}) \Omega^t_1 + \frac{6\eta_{out}^2}{\gamma\rho} \Omega^t_2 \\
     \Omega^{t+1}_2 &\leq (1-\frac{\gamma\rho}{2}+\frac{\tau(\lambda L_g)^2}{\gamma \rho}\eta^2)\Omega^t_2 + \frac{\tau(\lambda L_g)^2}{\gamma \rho}\gamma^2\rho^\prime \Omega^t_1 \\ & \quad+ \frac{\tau(\lambda L_g)^2}{\gamma \rho}\eta^2m \|\bar{s}^t \|^2 + 5(\lambda L_g)^2 C_{yz} \exp(-\frac{\mu K}{8L_g}),
    \end{align*}
\end{lemma}
where $C_{yz} = \max\{C_y, C_z\}$ and $\tau$ refers to bounded constants.
Lemma \ref{outer_loop_error} demonstrates that, with proper set of $\eta$ and $\gamma$, the model and tracker consensus errors are guaranteed to decrease. This provides the necessary conditions to ensure that the value function $\Omega^t$ progressively decreases over iterations, contributing to the final result of our analysis.

\begin{theorem}
    \label{outer_loop_convergence}
    Under Assumptions \ref{assmp_graph} and \ref{assump_smooth}, if we set 
    \begin{align*}
         \lambda\propto \mathcal{O}(l\kappa^3\epsilon^{-1}) \quad \eta_{out} \propto  \mathcal{O}(\gamma l^{-4} \kappa^{-6}\epsilon^{2}) \quad \gamma_{out} \propto  \mathcal{O}(\rho^2),
    \end{align*}
    Algorithm\ref{alg:C2DFB2} needs $O(l^{4}\kappa^{6}\rho^{-2}\epsilon^{-4}log(\epsilon^{-4}))$ first-order orcale calls to reach the $\epsilon$-stationary point.
\end{theorem}
 Based on above results, we can obtain the final communication complexities of our algorithms $O(l^{4}\kappa^{6}\rho^{-2}\epsilon^{-4}log(\epsilon^{-4}))$.

\begin{figure*}[t]
  \centering
  \includegraphics[width=\linewidth, height=6.5cm]{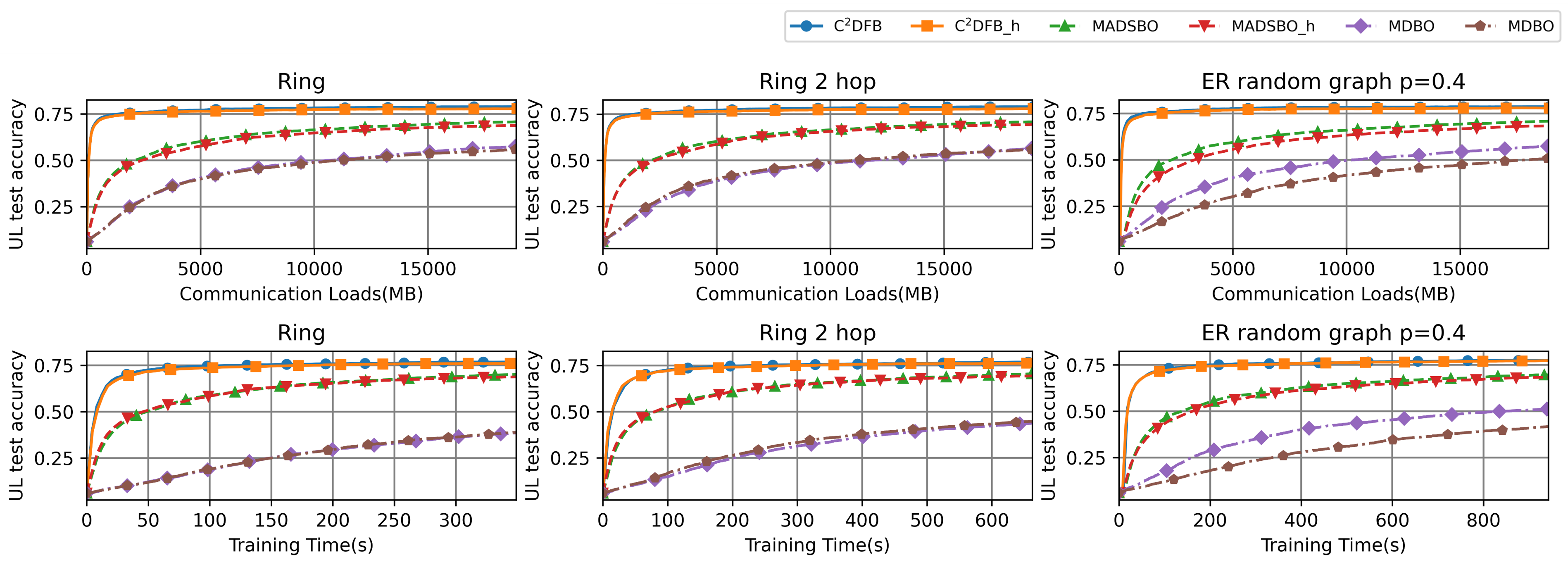}
  \caption{Comparison of upper-level test accuracy versus communication loads and training times for $\text{C}^2$DFB, MADSBO and MDBO under three topology on Coefficient Tuning task. The 'h' notation represents a heterogeneous data distribution across 10 clients, with a heterogeneity level set to 0.8 in the experiment.}
  \label{fig:heter}
\end{figure*}

% \begin{figure*}[t]
%   \centering
%   \includegraphics[width=\linewidth, height=4.5cm]{./pic/I2reg_sensitive.png}
%   \caption{Sensitive studies of $\text{C}^2$DFB, (1) varying the number of inner loops \( K \) (left), (2) varying the compression ratio (middle), and (3) varying the multiplier $\sigma$(right).}
%   \label{fig:sensitive}
% \end{figure*}
% since remaining space is limited, I will put this in the supplementary material

\begin{figure*}[t]
  \centering
  \includegraphics[width=\linewidth, height=4cm]{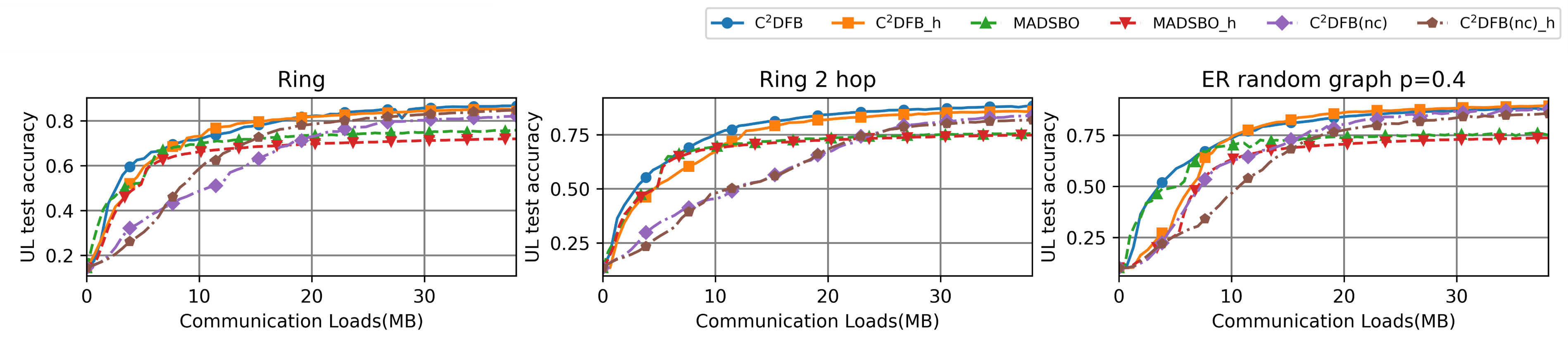}
  \caption{Upper-level test loss comparison versus communication loads for $\text{C}^2$DFB, MADSBO and a naive compression version of $\text{C}^2$DFB under three topology on Hyper Representation task. The 'h' notation represents a heterogeneous data distribution across 10 clients, with a heterogeneity level set to 0.8 in the experiment.}
  \label{fig:heter2}
\end{figure*}

\section{Experiments}\label{sec:exp}
\subsection{Coefficient Tunning on 20 Newsgroup}
We evaluate our algorithm by processing a hyperparameter tuning task using the 20 Newsgroups dataset, following the methodology outlined in \citet{NEURIPS2022_6dddcff5, kong2024decentralizedbileveloptimizationgraphs}. This dataset contains 101,631 features, making the computation and storage of the Jacobian/Hessian matrix impractical. In this task, %the upper-level (UL) and lower-level (LL) 
UL and LL functions are defined as follows:

%The dataset has a feature dimension of 101631 such that computing and storing Jacobian/Hessian matrix is impractical. In this task, the UL and LL function takes the form of
\begin{align*}
    f_i(x, y) &:= \frac{1}{|\mathcal{D}_{\text{val}}|} \sum_{(a_i, b_i) \in \mathcal{D}_{\text{val}}} \ell(\langle a_i, y \rangle, b_i),\\
    g_i(x, y) &:= \frac{1}{|\mathcal{D}_{\text{tr}}|} \sum_{(a_i, b_i) \in \mathcal{D}_{\text{tr}}} \!\!\!\!\!\! \ell(\langle a_i, y \rangle, b_i) + y^\top diag(\exp(x)) y,
\end{align*}
where $\ell(\cdot, \cdot)$ denotes the cross-entropy loss. $\mathcal{D}{\text{val}}$ and $\mathcal{D}{\text{tr}}$ represent the validation and training datasets, respectively. In this context, we employ a linear classifier as the model, with the objective of optimizing its coefficients through bilevel optimization.

For this experiment, we set the learning rates for both the outer and inner loops to 1, the consensus step to 0.5, and $\sigma$ to 10. The experiment was conducted on a network of 10 nodes using PyTorch multiprocessing. We evaluate three network topologies: a ring topology with each node linked to its two immediate neighbors; a 2-hop topology, connecting nodes to their neighbors' neighbors; and an Erdos–Renyi (ER) topology, randomly forming edges between nodes with a probability of $p=0.4$. Our results, depicted in Figure \ref{fig:heter}
, show that C$^{2}$DFB achieves faster convergence than these baselines across all topologies. Additionally, we assess performance in a data-heterogeneous setting where 80\%$(h)$ of each class's data is allocated to a specific client, with the remaining distributed among others.

% Figure \ref{fig:heter} compares our algorithm to the double-loop algorithm MADSBO \cite{pmlr-v202-chen23n} and the single-loop bilevel algorithm MDBO \cite{pmlr-v206-gao23a} across three different topologies. The results demonstrate that our algorithm converges more rapidly than baseline methods in all tested topologies. Additionally, we evaluate the performance under a data heterogeneous setting controlled by a parameter $h$. In this setting, a portion $h=0.8$ of the data from each class is allocated to a specific client, while the remainder is distributed among the other clients.

\begin{table}[!ht]
    \centering
    \caption{Comparison of Communication Volume and Training Time for Various Algorithms on the 20 Newsgroups Dataset over ring topology with heterogeneous data distribution, achieving 70\% Test Accuracy.}
    \begin{tabular}{lll}
    \hline\hline
        \textbf{Algo.} & \textbf{Comm. Vol. (MB)} & \textbf{Train. Time (s)} \\ \hline
        \textbf{C}$^{2}$\textbf{DFB} & 378.20 & 96.10 \\ 
        MADSBO & 24466.75 & 830.35 \\ 
        MDBO & 98463.75 & 9810.91  \\ 
        \hline\hline
    \end{tabular}
    \label{comm&comp}
\end{table}

% \textbf{C}$^{2}$\textbf{DFB} & 330.93 & 80.80 \\ 
% MADSBO & 16078.15 & 622.31 \\ 
% MDBO & 78890.35 & 8156.84 \\ 

% \textbf{C}$^{2}$\textbf{DFB\_h} & 378.20 & 96.10 \\ 
% MADSBO\_h & 24466.75 & 830.35 \\ 
% MDBO\_h & 98463.75 & 9810.91 \\ 
  
Table \ref{comm&comp} compares the communication volume and training time of the C$^{2}$DFB algorithm with baselines (MADSBO and MDBO) on the 20 Newsgroups dataset for achieving 70\% test accuracy over a ring topology under heterogeneous distribution. C$^{2}$DFB significantly outperforms the baselines, requiring only 387 MB of communication—about 260 times less than MDBO's 98,464 MB. Additionally, its training time is markedly reduced at 96 seconds, compared to MDBO’s 7,831 seconds. These results highlight the efficiency of C$^{2}$DFB in both communication and computation.

\subsection{Hyper-Representation Learning}
We further evaluate our algorithm on a hyper-representation learning task, as described in \citet{pmlr-v162-tarzanagh22a}. This task aims to enhance the model's feature representation and improve its performance on downstream tasks. The outer objective focuses on refining the model backbone, while the inner objective optimizes a smaller, task-specific head responsible for classification on the training data. We implement this task using a three-layer multilayer perceptron on the MNIST dataset, where the outer optimization targets the hidden units with 81,902 parameters, and the inner optimization focuses on the classification head, comprising approximately 640 parameters.

To demonstrate the impact of our reference point compression technique, we introduce a baseline variant, C$^{2}$DFB(nc), which simply compresses the transmitted parameters and compensates for accumulated compression errors from previous communication rounds.

Our results, illustrated in Fig. \ref{fig:heter2}, highlight the effectiveness and stability of our algorithm across three topologies, under both homogeneous and heterogeneous data settings. Compared to MADSBO, our algorithm exhibits faster convergence and higher final accuracy over 2,000 rounds. Moreover, unlike the naive variant, our algorithm demonstrates greater stability and speed during the training process. This improvement is likely due to the fact that our algorithm aligns closely with the centralized version of the Fully First-Order Bilevel method from a global perspective.

% We also compare our $\text{C}^2$DFB with MADSBO and a naive-compression version of our algorithm that transfers the compressed parameter across networks with error compensation on a hyper-representation learning task\cite{pmlr-v162-tarzanagh22a}. Specifically, this task focuses on improving the model's feature representation and its performance on downstream tasks. The outer objective refines the model's backbone, and the inner one is then concerned with optimizing a smaller, task-specific head that is responsible for performing classification on training data.

% We implement this task using a three layer multilayer perceptron on MNIST dataset, where the outer optimization targets at the hidden unit with 81902 parameters and the inner optimization is focused on the classification head comprising around 640 parameters. Our results, as depicted in \ref{fig:heter2}, underscore the effectiveness and stability of our algorithm, both under i.i.d. and non i.i.d. setting. Additionally, we include a naive-compress version which simply compresses the parameter with error compression, demonstrating an inferior performance in convergence and stability than our proposed referecence point strategy. 

\section{Conclusion}
This paper presents the C$^{2}$DFB algorithm, a novel first-order gradient-based method developed for decentralized bilevel optimization. It addresses both computational and communication challenges by employing a reference point-based compression strategy that substantially lowers communication requirements. Our theoretical analysis validates the convergence efficiency of C$^{2}$DFB. Extensive experimental evaluations on hyperparameter tuning and hyper-representation learning tasks across diverse topologies and data distributions show that C$^{2}$DFB outperforms existing approaches.

\begingroup
\small  % \footnotesize, \scriptsize, \small
\bibliography{ref}

\begin{thebibliography}{26}
\providecommand{\natexlab}[1]{#1}

\bibitem[{Bonawitz et~al.(2021)Bonawitz, Kairouz, McMahan, and
  Ramage}]{10.1145/3494834.3500240}
Bonawitz, K.; Kairouz, P.; McMahan, B.; and Ramage, D. 2021.
\newblock Federated Learning and Privacy: Building privacy-preserving systems
  for machine learning and data science on decentralized data.
\newblock \emph{Queue}, 19(5): 87–114.

\bibitem[{Chen, Ma, and
  Zhang(2023)}]{chen2023nearoptimalnonconvexstronglyconvexbileveloptimization}
Chen, L.; Ma, Y.; and Zhang, J. 2023.
\newblock Near-Optimal Nonconvex-Strongly-Convex Bilevel Optimization with
  Fully First-Order Oracles.
\newblock arXiv:2306.14853.

\bibitem[{Chen et~al.(2023)Chen, Huang, Ma, and
  Balasubramanian}]{pmlr-v202-chen23n}
Chen, X.; Huang, M.; Ma, S.; and Balasubramanian, K. 2023.
\newblock Decentralized Stochastic Bilevel Optimization with Improved
  per-Iteration Complexity.
\newblock In Krause, A.; Brunskill, E.; Cho, K.; Engelhardt, B.; Sabato, S.;
  and Scarlett, J., eds., \emph{Proceedings of the 40th International
  Conference on Machine Learning}, volume 202 of \emph{Proceedings of Machine
  Learning Research}, 4641--4671. PMLR.

\bibitem[{Dong et~al.(2024)Dong, Ma, Yang, and
  Yin}]{dong2024singleloopalgorithmdecentralizedbilevel}
Dong, Y.; Ma, S.; Yang, J.; and Yin, C. 2024.
\newblock A Single-Loop Algorithm for Decentralized Bilevel Optimization.
\newblock arXiv:2311.08945.

\bibitem[{Finn, Abbeel, and Levine(2017)}]{pmlr-v70-finn17a}
Finn, C.; Abbeel, P.; and Levine, S. 2017.
\newblock Model-Agnostic Meta-Learning for Fast Adaptation of Deep Networks.
\newblock In Precup, D.; and Teh, Y.~W., eds., \emph{Proceedings of the 34th
  International Conference on Machine Learning}, volume~70 of \emph{Proceedings
  of Machine Learning Research}, 1126--1135. PMLR.

\bibitem[{Franceschi et~al.(2018)Franceschi, Frasconi, Salzo, Grazzi, and
  Pontil}]{pmlr-v80-franceschi18a}
Franceschi, L.; Frasconi, P.; Salzo, S.; Grazzi, R.; and Pontil, M. 2018.
\newblock Bilevel Programming for Hyperparameter Optimization and
  Meta-Learning.
\newblock In Dy, J.; and Krause, A., eds., \emph{Proceedings of the 35th
  International Conference on Machine Learning}, volume~80 of \emph{Proceedings
  of Machine Learning Research}, 1568--1577. PMLR.

\bibitem[{Gao, Gu, and Thai(2023)}]{pmlr-v206-gao23a}
Gao, H.; Gu, B.; and Thai, M.~T. 2023.
\newblock On the Convergence of Distributed Stochastic Bilevel Optimization
  Algorithms over a Network.
\newblock In Ruiz, F.; Dy, J.; and van~de Meent, J.-W., eds., \emph{Proceedings
  of The 26th International Conference on Artificial Intelligence and
  Statistics}, volume 206 of \emph{Proceedings of Machine Learning Research},
  9238--9281. PMLR.

\bibitem[{Gao et~al.(2022)Gao, Ye, Yin, Zeng, and Zhang}]{pmlr-v162-gao22j}
Gao, L.~L.; Ye, J.; Yin, H.; Zeng, S.; and Zhang, J. 2022.
\newblock Value Function based Difference-of-Convex Algorithm for Bilevel
  Hyperparameter Selection Problems.
\newblock In Chaudhuri, K.; Jegelka, S.; Song, L.; Szepesvari, C.; Niu, G.; and
  Sabato, S., eds., \emph{Proceedings of the 39th International Conference on
  Machine Learning}, volume 162 of \emph{Proceedings of Machine Learning
  Research}, 7164--7182. PMLR.

\bibitem[{Ghadimi and
  Wang(2018)}]{ghadimi2018approximationmethodsbilevelprogramming}
Ghadimi, S.; and Wang, M. 2018.
\newblock Approximation Methods for Bilevel Programming.
\newblock arXiv:1802.02246.

\bibitem[{Hong et~al.(2023)Hong, Wai, Wang, and Yang}]{doi:10.1137/20M1387341}
Hong, M.; Wai, H.-T.; Wang, Z.; and Yang, Z. 2023.
\newblock A Two-Timescale Stochastic Algorithm Framework for Bilevel
  Optimization: Complexity Analysis and Application to Actor-Critic.
\newblock \emph{SIAM Journal on Optimization}, 33(1): 147--180.

\bibitem[{Ji et~al.(2022)Ji, Liu, Liang, and Ying}]{NEURIPS2022_1413947e}
Ji, K.; Liu, M.; Liang, Y.; and Ying, L. 2022.
\newblock Will Bilevel Optimizers Benefit from Loops.
\newblock In Koyejo, S.; Mohamed, S.; Agarwal, A.; Belgrave, D.; Cho, K.; and
  Oh, A., eds., \emph{Advances in Neural Information Processing Systems},
  volume~35, 3011--3023. Curran Associates, Inc.

\bibitem[{Ji, Yang, and Liang(2021)}]{pmlr-v139-ji21c}
Ji, K.; Yang, J.; and Liang, Y. 2021.
\newblock Bilevel Optimization: Convergence Analysis and Enhanced Design.
\newblock In Meila, M.; and Zhang, T., eds., \emph{Proceedings of the 38th
  International Conference on Machine Learning}, volume 139 of
  \emph{Proceedings of Machine Learning Research}, 4882--4892. PMLR.

\bibitem[{Koloskova, Lin, and Stich(2021)}]{NEURIPS2021_5f25fbe1}
Koloskova, A.; Lin, T.; and Stich, S.~U. 2021.
\newblock An Improved Analysis of Gradient Tracking for Decentralized Machine
  Learning.
\newblock In Ranzato, M.; Beygelzimer, A.; Dauphin, Y.; Liang, P.; and Vaughan,
  J.~W., eds., \emph{Advances in Neural Information Processing Systems},
  volume~34, 11422--11435. Curran Associates, Inc.

\bibitem[{Kong et~al.(2024)Kong, Zhu, Lu, Huang, and
  Yuan}]{kong2024decentralizedbileveloptimizationgraphs}
Kong, B.; Zhu, S.; Lu, S.; Huang, X.; and Yuan, K. 2024.
\newblock Decentralized Bilevel Optimization over Graphs: Loopless Algorithmic
  Update and Transient Iteration Complexity.
\newblock arXiv:2402.03167.

\bibitem[{Kwon et~al.(2023)Kwon, Kwon, Wright, and Nowak}]{pmlr-v202-kwon23c}
Kwon, J.; Kwon, D.; Wright, S.; and Nowak, R.~D. 2023.
\newblock A Fully First-Order Method for Stochastic Bilevel Optimization.
\newblock In Krause, A.; Brunskill, E.; Cho, K.; Engelhardt, B.; Sabato, S.;
  and Scarlett, J., eds., \emph{Proceedings of the 40th International
  Conference on Machine Learning}, volume 202 of \emph{Proceedings of Machine
  Learning Research}, 18083--18113. PMLR.

\bibitem[{Li, Huang, and Huang(2023)}]{NEURIPS2023_04bd683d}
Li, J.; Huang, F.; and Huang, H. 2023.
\newblock Communication-Efficient Federated Bilevel Optimization with Global
  and Local Lower Level Problems.
\newblock In Oh, A.; Naumann, T.; Globerson, A.; Saenko, K.; Hardt, M.; and
  Levine, S., eds., \emph{Advances in Neural Information Processing Systems},
  volume~36, 1326--1338. Curran Associates, Inc.

\bibitem[{Liao et~al.(2022)Liao, Li, Huang, and Pu}]{9789732}
Liao, Y.; Li, Z.; Huang, K.; and Pu, S. 2022.
\newblock A Compressed Gradient Tracking Method for Decentralized Optimization
  With Linear Convergence.
\newblock \emph{IEEE Transactions on Automatic Control}, 67(10): 5622--5629.

\bibitem[{Liu et~al.(2022)Liu, Ye, Wright, Stone, and
  Liu}]{NEURIPS2022_6dddcff5}
Liu, B.; Ye, M.; Wright, S.; Stone, P.; and Liu, Q. 2022.
\newblock BOME! Bilevel Optimization Made Easy: A Simple First-Order Approach.
\newblock In Koyejo, S.; Mohamed, S.; Agarwal, A.; Belgrave, D.; Cho, K.; and
  Oh, A., eds., \emph{Advances in Neural Information Processing Systems},
  volume~35, 17248--17262. Curran Associates, Inc.

\bibitem[{Lu et~al.(2022)Lu, Zeng, Cui, Squillante, Horesh, Kingsbury, Liu, and
  Hong}]{NEURIPS2022_c5cf13bf}
Lu, S.; Zeng, S.; Cui, X.; Squillante, M.; Horesh, L.; Kingsbury, B.; Liu, J.;
  and Hong, M. 2022.
\newblock A Stochastic Linearized Augmented Lagrangian Method for Decentralized
  Bilevel Optimization.
\newblock In Koyejo, S.; Mohamed, S.; Agarwal, A.; Belgrave, D.; Cho, K.; and
  Oh, A., eds., \emph{Advances in Neural Information Processing Systems},
  volume~35, 30638--30650. Curran Associates, Inc.

\bibitem[{Nichol, Achiam, and
  Schulman(2018)}]{nichol2018firstordermetalearningalgorithms}
Nichol, A.; Achiam, J.; and Schulman, J. 2018.
\newblock On First-Order Meta-Learning Algorithms.
\newblock arXiv:1803.02999.

\bibitem[{Qin, Song, and Jiang(2023)}]{Qin_2023_CVPR}
Qin, X.; Song, X.; and Jiang, S. 2023.
\newblock Bi-Level Meta-Learning for Few-Shot Domain Generalization.
\newblock In \emph{Proceedings of the IEEE/CVF Conference on Computer Vision
  and Pattern Recognition (CVPR)}, 15900--15910.

\bibitem[{Shi et~al.(2023)Shi, Shen, Wei, Sun, Yuan, Wang, and
  Tao}]{pmlr-v202-shi23d}
Shi, Y.; Shen, L.; Wei, K.; Sun, Y.; Yuan, B.; Wang, X.; and Tao, D. 2023.
\newblock Improving the Model Consistency of Decentralized Federated Learning.
\newblock In Krause, A.; Brunskill, E.; Cho, K.; Engelhardt, B.; Sabato, S.;
  and Scarlett, J., eds., \emph{Proceedings of the 40th International
  Conference on Machine Learning}, volume 202 of \emph{Proceedings of Machine
  Learning Research}, 31269--31291. PMLR.

\bibitem[{Tang et~al.(2018)Tang, Gan, Zhang, Zhang, and
  Liu}]{NEURIPS2018_44feb009}
Tang, H.; Gan, S.; Zhang, C.; Zhang, T.; and Liu, J. 2018.
\newblock Communication Compression for Decentralized Training.
\newblock In Bengio, S.; Wallach, H.; Larochelle, H.; Grauman, K.;
  Cesa-Bianchi, N.; and Garnett, R., eds., \emph{Advances in Neural Information
  Processing Systems}, volume~31. Curran Associates, Inc.

\bibitem[{Tarzanagh et~al.(2022)Tarzanagh, Li, Thrampoulidis, and
  Oymak}]{pmlr-v162-tarzanagh22a}
Tarzanagh, D.~A.; Li, M.; Thrampoulidis, C.; and Oymak, S. 2022.
\newblock {F}ed{N}est: Federated Bilevel, Minimax, and Compositional
  Optimization.
\newblock In Chaudhuri, K.; Jegelka, S.; Song, L.; Szepesvari, C.; Niu, G.; and
  Sabato, S., eds., \emph{Proceedings of the 39th International Conference on
  Machine Learning}, volume 162 of \emph{Proceedings of Machine Learning
  Research}, 21146--21179. PMLR.

\bibitem[{Yang, Zhang, and Wang(2022)}]{NEURIPS2022_01db36a6}
Yang, S.; Zhang, X.; and Wang, M. 2022.
\newblock Decentralized Gossip-Based Stochastic Bilevel Optimization over
  Communication Networks.
\newblock In Koyejo, S.; Mohamed, S.; Agarwal, A.; Belgrave, D.; Cho, K.; and
  Oh, A., eds., \emph{Advances in Neural Information Processing Systems},
  volume~35, 238--252. Curran Associates, Inc.

\bibitem[{Yang, Xiao, and Ji(2023)}]{NEURIPS2023_686a3f32}
Yang, Y.; Xiao, P.; and Ji, K. 2023.
\newblock SimFBO: Towards Simple, Flexible and Communication-efficient
  Federated Bilevel Learning.
\newblock In Oh, A.; Naumann, T.; Globerson, A.; Saenko, K.; Hardt, M.; and
  Levine, S., eds., \emph{Advances in Neural Information Processing Systems},
  volume~36, 33027--33040. Curran Associates, Inc.

\end{thebibliography}
\endgroup

\newpage
\onecolumn
\section*{Supplementary Material}
\appendix

\renewcommand{\thesection}{\Alph{section}}

This supplementary material is organized as follows: Appendix \ref{app:not}  provides a summary of the notations used throughout the appendix and reviews the key objectives addressed in this paper. Appendix \ref{app:proof} presents the convergence proofs for both the inner and outer loops. Appendix \ref{app:exp} provides additional experimental details, including a comprehensive description of the experimental setup and sensitivity analyses on various hyperparameters.

\section{Notations}\label{app:not}
\subsection{Problem Recap}
In this paper, we tackle the following optimization problem:

\begin{equation}
\min_{x \in \mathbb{R}^d_x} \psi(x) := \min_{y \in \mathbb{R}^{d_y}} \left\{\frac{1}{m}\sum_{i=1}^m{f_i(x, y)}\right\},
\label{eq2}
\end{equation}

where

\begin{align*}
y^\ast(x) &= \underset{y \in \mathbb{R}^{d_y}}{\arg\min} \left\{\frac{1}{m}\sum_{i=1}^m{ g_i(x, y)}\right\}.
\end{align*}

To achieve a stationary point of equation \ref{eq2} without resorting to second-order information, we reformulate the problem as follows:
\begin{equation}
\begin{aligned}
\min_{x \in \mathbb{R}^d_x} \psi_\lambda^\ast(x)
&:= \min_{y \in \mathbb{R}^{d_y}} \frac{1}{m}\sum_{i=1}^m f_i(x,y) + \lambda \bigg( g_i(x,y) - g^\ast(x) \bigg), \\
g^\ast(x) &= \min_{z \in \mathbb{R}^{d_y}} \frac{1}{m}\sum_{i=1}^m g_i(x,z).
\end{aligned}
\label{reform}
\end{equation}

To solve this problem in a centralized manner, in the inner loop, we approximate $y_{\lambda}^*(\bar{x})$ and $y^*(\bar{x})$, defined as,
\begin{align*}
y_\lambda^\ast(\bar{x}^t) &= \arg \min_{y \in \mathbb{R}^{d_y}} \left\{\frac{1}{m}\sum_{i=1}^m f_i(\bar{x}^t, y) + \lambda g_i(\bar{x}^t, y)\right\}, \\
y^\ast(\bar{x}^t) &= \arg \min_{y \in \mathbb{R}^{d_y}} \frac{1}{m}\sum_{i=1}^m g_i(\bar{x}, y).
\end{align*}

The gradient of equation \ref{reform} is given by:

\begin{equation}
\begin{aligned}
\nabla \psi^\ast_{\lambda}(x) &= \nabla_x \psi_{\lambda}(x, y_{\lambda}^\ast(\bar{x})) + \nabla_y y_{\lambda}^\ast(\bar{x}) \nabla_y \psi_{\lambda}(x, y_{\lambda}^\ast(\bar{x})) \\
&= \frac{1}{m}\sum_{i=1}^m \nabla_x  f_i(x, y_{\lambda}^\ast(x)) + \lambda \bigg(\nabla_x g_i(x, y_{\lambda}^\ast(\bar{x})) - \nabla_x g_i(x, y^\ast(\bar{x}))\bigg).
\end{aligned}
\end{equation}

In the decentralized setting, the optimization shifts based on the local outer parameters, with the optimal solutions denoted as:

\begin{align*}
\tilde{y}_\lambda^\ast &= \arg \min_{y \in \mathbb{R}^{d_y}} \left\{\frac{1}{m}\sum_{i=1}^m f_i(x_i^t, y) + \lambda g_i(x_i^t, y)\right\}, \
\tilde{y}^\ast &= \arg \min_{y \in \mathbb{R}^{d_y}} \frac{1}{m}\sum_{i=1}^m g_i(x_i^t, y).
\end{align*}

This is achieved by performing $K$ steps of gradient descent, resulting in the updates $\mathbf{y}_t^K$ and  $\mathbf{z}_t^K$. 

Upon completing the inner loop, the outer loop updates the parameters through the following expression:

\begin{align*}
\widetilde{\hat{\nabla}_x\psi_i(x_i^t)} &= \nabla_x f_i(x_i,y_i^K)+ \lambda\left(\nabla_x g_i(x_i, y_i^K)- \nabla_x g_i(x_i, z_i^K)\right).
\end{align*}

\subsection{Notations involved}
\textbf{\noindent Models}
\\
We denote the average of the parameters as follows:
\begin{align*}
   \text{Outer parameter } \bar{x} = \frac{1}{m} \sum_{i=1}^m x_i, \quad
   \text{Inner parameter } \bar{y} = \frac{1}{m} \sum_{i=1}^m y_i, \quad
   \bar{z} = \frac{1}{m} \sum_{i=1}^m z_i.
\end{align*}

A stacked version of the global parameters is defined as:
\begin{align*}
    \mathbf{x} &= \begin{bmatrix}
        x_1,\\
        x_2,\\
        \cdots\\ 
        x_m
    \end{bmatrix}^\top \in \mathbb{R}^{m \times d_x}.
    \quad
    \mathbf{y} &= \begin{bmatrix}
        y_1,\\
        y_2,\\
        \cdots\\ 
        y_m
    \end{bmatrix}^\top \in \mathbb{R}^{m \times d_y}
    \quad
    \mathbf{z} &= \begin{bmatrix}
        z_1,\\
        z_2,\\
        \cdots\\ 
        z_m
    \end{bmatrix}^\top \in \mathbb{R}^{m \times d_y},
\end{align*}

Additionally, let $\mathbf{1} = [1, \cdots, 1]^\top \in \mathbb{R}^m$, and $\|\cdot\|$ denote the 2-norm for vectors and the Frobenius norm for a matrix.

In this paper, the tracker is typically denoted as $s_{\cdot}$, and the reference point as $\hat{\cdot}$, where $\cdot$ represents any of the parameters. For the inner loop, the parameter on the $i$-th client at iteration $t$ of inner loop $k$ is represented as $(\cdot^k)^t$. Across the proofs, we omit the outer loop iterations when the analysis is confined to a single iteration.

\noindent \textbf{Functions}
\\
This paper involves several key functions, denoted as follows:
\begin{itemize}
\item $f_i$ and $g_i$ represent the upper-level and lower-level functions, respectively, in the bilevel optimization problem.
\item $\psi(x)$ denotes the initial hyper-objective function with the optimal $y$.
\item $\psi_\lambda^\ast(x)$ refers to the Lagrangian reformulation with the optimal $y$ and $z$.
\item $h_i(x,y) := \left( f_i(x,y) + \lambda g_i(x,y) \right)$ represents the objective function used to determine the upper-level (UL) parameter in the Lagrangian reformulation.
\item $\nabla_x \psi_\lambda^\ast(x)$ denotes the hypergradient of the reformulated problem with the optimal $y$ and $z$.
\item $\widetilde{\hat{\nabla}_x\psi_i(x_i^t)}$ represents the approximate hypergradient obtained after $K$ steps of local model optimization, where the models solve $\min_{y \in \mathbb{R}^{d_y}}\frac{1}{m}\sum_{i=1}^m h(x_i, y)$ and $\min_{z \in \mathbb{R}^{d_y}}\frac{1}{m}\sum_{i=1}^m g(x_i, z)$.
\end{itemize}
\noindent \textbf{Constants}
\\
This paper introduces several hyperparameters based on assumptions and the proposed algorithm. Their meanings are as follows:
\begin{itemize}
\item For the algorithm, the step size is denoted by $\eta$, the mixing constants by $\gamma$, and the Lagrange multipliers by $\lambda$. The goal is to recursively update the model until the gradient norm of $\psi(x)$ is reduced to a sufficiently small positive number $\epsilon$.
\item In terms of assumptions related to the smoothness of the objective function, $C$ denotes the Lipschitz continuity of the objective function, $l$ represents the Lipschitz constant for the gradients, and $\rho$ refers to the Lipschitz continuity of the Hessian. The condition number is denoted by $\kappa$.
\item For assumptions related to compression, any compressor $\mathcal{Q}$ is associated with a compression parameter $\delta_c$.
\item Regarding network-related assumptions, $\rho$ denotes the spectral gap of the connecting matrix $\mathbf{W}$.
\end{itemize}

\section{Proofs}\label{app:proof}

\subsection{Assumptions}

\begin{assumption}
    \label{assump:smooth}
    Consider the upper-level (UL) and lower-level (LL) problems in the bilevel optimization framework. We assume the following conditions:
    \begin{enumerate}
        \item The upper-level function $f_i(x, y)$ is $C_f$-Lipschitz continuous in $y$, with $L_f$-Lipschitz continuous gradients and $\rho_f$-Lipschitz continuous Hessians.
        \item The lower-level function $g_i(x, y)$ is $\mu$-strongly convex in $y$, with $L_g$-Lipschitz continuous gradients and $\rho_g$-Lipschitz continuous Hessians.
    \end{enumerate}
\end{assumption}

\begin{definition}
    Given Assumption \ref{assump:smooth}, we define $l = \max\{C_f, L_f, L_g, \rho_g\}$ and the condition number as $\kappa = l/\mu$. 
\end{definition}

\begin{assumption}
    \label{assump:graph}
    To model real-world network scenarios, we assume that the graph $\mathcal{G} = (\mathcal{V}, \mathcal{E})$ is connected and undirected, represented by a mixing matrix $\mathbf{W} \in \mathbb{R}^{m \times m}$. The matrix $\mathbf{W}$ satisfies the following properties:
    \begin{itemize}
        \item[1)] $w_{ij} > 0$ if $(i,j) \in \mathcal{E}$, and $w_{ij} = 0$ otherwise.
        \item[2)] $\mathbf{W}$ is doubly stochastic, i.e., $\mathbf{W} = \mathbf{W}^\top$, $\sum_{i=1}^m w_{ij} = 1$, and $\sum_{j=1}^m w_{ij} = 1$.
        \item[3)] The eigenvalues of $\mathbf{W}$ satisfy $\lambda_m \leq \ldots \leq \lambda_2 \leq \lambda_1 = 1$ and $\nu = \max\{|\lambda_2|, |\lambda_m|\} < 1$.
    \end{itemize}
\end{assumption}

\begin{definition}
    For a gossip mixing matrix $\mathbf{W}$ satisfying Assumption \ref{assump:graph}, the spectral gap is defined as $\rho \equiv 1 - \delta_\rho$, where $\delta_\rho \equiv \max\{|\lambda_2(\mathbf{W})|, |\lambda_m(\mathbf{W})|\}$ is the second-largest eigenvalue in magnitude.
\end{definition}

\subsection{Propositions}

\begin{proposition}
\label{prop:biased_compressor}
If $\mathcal{Q}$ is an unbiased contractive compressor satisfying $\mathbb{E} \left[ \|\mathcal{Q}(A) - A\|^2 \right] \leq (1 - \delta_c)\|A\|^2$, then the biased compressor $\mathcal{Q}^{\prime}(A) = \frac{\mathcal{Q}(A)}{2-\delta_c}$ is also contractive with a contraction factor $\delta_c^\prime = \frac{1}{2-\delta_c}$.
\end{proposition}

\begin{proof}
    Let $\mathcal{Q}$ be an unbiased contractive compressor, such that
    \begin{equation}
        \mathbb{E} \left[ \|\mathcal{Q}(A) - A\|^2 \right] \leq (1 - \delta_c)\|A\|^2 \quad \text{and} \quad \mathbb{E} \left[ \mathcal{Q}(A) \right] = A.
        \label{unbias2bias}
    \end{equation}
    Construct a new compressor $\mathcal{Q}^{\prime}(A) = \frac{\mathcal{Q}(A)}{2-\delta_c}$. Since $\mathbb{E}\left[\mathcal{Q}^{\prime}(A)\right] = \frac{A}{2-\delta_c}$, $\mathcal{Q}^{\prime}$ is a biased compressor. The contractive property is given by:
    \begin{align*}
        \mathbb{E} \left[ \|\mathcal{Q}^{\prime}(A) - A\|^2 \right] 
        &= \mathbb{E} \left[ \left\|\frac{\mathcal{Q}(A) - A}{2-\delta_c} - \frac{(1-\delta_c)A}{2-\delta_c}\right\|^2 \right] \\
        &= \frac{1}{(2-\delta_c)^2} \mathbb{E} \left[\|\mathcal{Q}(A)- A \|^2\right] \\
        &\quad - \frac{2}{2-\delta_c} \mathbb{E} \left[ \langle \mathcal{Q}(A)- A, (1-\delta_c)A \rangle \right] \\
        &\quad + \frac{(1-\delta_c)^2}{(2-\delta_c)^2}\|A\|^2 \\
        &\leq \frac{1-\delta_c}{(2-\delta_c)^2}\|A\|^2 \\
        &\quad - \frac{2}{2-\delta_c}\langle \mathbb{E}\left[\mathcal{Q}(A)\right]- A, (1-\delta_c)A \rangle \\
        &\quad + \frac{(1-\delta_c)^2}{(2-\delta_c)^2}\|A\|^2 \\
        &= \left(1 - \frac{1}{2-\delta_c}\right)\|A\|^2,
    \end{align*}
    where the last two lines use \eqref{unbias2bias}. This indicates that the biased compressor $\mathcal{Q}^{\prime}$ remains contractive with a contraction factor of $\delta_c^\prime= \frac{1}{2-\delta_c}$.
\end{proof}

\noindent\textbf{Smoothness of Various Objectives}
\begin{proposition}\cite{ghadimi2018approximationmethodsbilevelprogramming}
\label{prop:liptschiz of psi}
Under Assumption \ref{assump:smooth}, the hypergradient, uniquely defined as
\begin{align*}
   \nabla \psi(x) &= \left(\frac{1}{m}\sum_{i=1}^m \nabla_x f_i(x, y^*(x))\right) \\
   &\quad - \nabla_{xy} g(x, y^*(x)) \left[ \nabla_y^2 g(x, y^*(x)) \right]^{-1} \left(\frac{1}{m}\sum_{i=1}^m \nabla_y f_i(x, y^*(x))\right),
\end{align*}
and the hyper-objective $\varphi(x)$ is $L_{\varphi}$-Lipschitz continuous with respect to the gradient, where $L_{\varphi} = \mathcal{O}(\ell \kappa^3)$.
\end{proposition}

\begin{proposition}\cite{pmlr-v202-kwon23c,chen2023nearoptimalnonconvexstronglyconvexbileveloptimization}
    \label{prop:first-order}
    Under Assumption \ref{assump:smooth}, if $\lambda \geq 2L_f/\mu$, it holds that:
    \begin{itemize}
        \item[1)] $\|\nabla \psi_\lambda^\ast(x)-\nabla \psi(x)\| \leq \mathcal{O}(\frac{l\kappa^3}{\lambda})$ and $\|\psi^\ast_\lambda(x)- \psi(x)\| \leq \mathcal{O}(\frac{l\kappa^2}{\lambda})$
        \item[2)] $\psi_\lambda^\ast(x)$ has a gradient of $\mathcal{O}(\kappa^3)$.
        \item[3)] $h_i(x,y) := \left( f_i(x,y) + \lambda g_i(x,y) \right)$ is $\frac{\lambda \mu}{2}$ strongly convex in $y$.
    \end{itemize}
\end{proposition}

\noindent\textbf{Features of Decentralized Learning}
\begin{proposition}
\label{gradient_tracking}
Let $s_i^k$ be the gradient tracker for the $i$-th client, tracking the gradient $\nabla r (\bar{d})$. For any $k > 0$, the following holds:
\begin{align*}
    \bar{s}^k &= \frac{1}{n} \sum_{i=1}^m \nabla r (d_i^k), \\
    \bar{d}^{k+1} &= \bar{d}^{k} - \eta \frac{1}{n} \sum_{i=1}^m \nabla r (d_i^k).
\end{align*}
\end{proposition}

\begin{proposition}
\label{prop:spectral_gap_proposition}
Let $W$ be a doubly stochastic matrix with a spectral gap $\rho$ and let $\gamma \in (0, 1)$. Define the matrix $\tilde{W} = I + \gamma (W - I)$. Then $\tilde{W}$ retains a spectral gap of at least $\gamma \rho$, meaning that the second largest eigenvalue $\lambda_2(\tilde{W})$ satisfies:
\begin{align*}
1 - \lambda_2(\tilde{W}) \geq \gamma \rho.
\end{align*}
\end{proposition}

\begin{proof}
Since $\mathbf{W}$ is doubly stochastic, its largest eigenvalue is $\lambda_1(\mathbf{W}) = 1$ with the corresponding eigenvector $\mathbf{1}$. The eigenvalues of $\mathbf{W}$ are thus $1 = \lambda_1(\mathbf{W}) > \lambda_2(\mathbf{W}) \geq \cdots \geq \lambda_n(\mathbf{W})$.

When transforming $W$ into $\tilde{\mathbf{W}} = I + \gamma (\mathbf{W} - I)$, the eigenvalues of $\tilde{\mathbf{W}}$ are transformed according to:
\begin{align*}
\lambda_i(\tilde{\mathbf{W}}) = 1 - \gamma + \gamma \lambda_i(\mathbf{W}), \quad \forall i.
\end{align*}
This transformation shifts $\lambda_1(\mathbf{W}) = 1$ to $\lambda_1(\tilde{\mathbf{W}}) = 1$. For $i \geq 2$, we have:
\begin{align*}
\lambda_i(\tilde{\mathbf{W}}) &= 1 - \gamma + \gamma \lambda_i(\mathbf{W}) \\
&= 1 - \gamma (1 - \lambda_i(\mathbf{W})).
\end{align*}
Hence, the spectral gap of $\tilde{\mathbf{W}}$, which is the difference $1 - \lambda_2(\tilde{\mathbf{W}})$, becomes:
\begin{align*}
1 - \lambda_2(\tilde{\mathbf{W}}) &= \gamma (1 - \lambda_2(\mathbf{W})) \geq \gamma \rho.
\end{align*}
Therefore, the spectral gap of $\tilde{\mathbf{W}}$ is at least $\gamma \rho$, as required.
\end{proof}

\subsection{Lemmas for the Inner Loop}

For the inner loop, we design the following Lyapunov functions:
\begin{equation}
    \begin{aligned}
         \text{Model Optimization error: } \Omega^k_0 & := \|\bar{d}^k - d^\ast \|^2, \\
        \text{Model compression error: } \Omega^k_1 & := \|\mathbf{d}^k - \hat{\mathbf{d}}^k \|^2, \\
        \text{Model consensus error: } \Omega^k_2 & := \|\mathbf{d}^k - \textbf{1}\bar{d}^k\|^2, \\
        \text{Model compression error: } \Omega^k_3 & := \|\mathbf{s}_d^k - \hat{\mathbf{s}}_d^k \|^2, \\
        \text{Tracker consensus error: } \Omega^k_4 & := \|\mathbf{s}^k - \textbf{1}\bar{s}^k \|^2, 
    \end{aligned}
    \label{valuefunc_innerloop}
\end{equation}

The inner loop can be considered as a single-machine problem, specifically minimizing $\min_{d \in \mathbb{R}^{d_y}} \frac{1}{m}\sum_{i=1}^m r_i(d)$. . Based on the previous study by \citet{9789732}, the value function defined in equation \ref{valuefunc_innerloop} can be recursively bounded using the following lemma.
\begin{lemma}
\label{lem:inner recur}
 Under assumptions \ref{assump:smooth}-\ref{assump:graph}, let $\Omega^k_0, \Omega^k_1,  \Omega^k_2, \Omega^k_3, \Omega^k_4$ defined in Eq.\ref{valuefunc_outerloop}, we have the following recursions,
    \begin{align*}
    \Omega^{k+1}_0 & \leq (1-\eta\mu)\Omega^{k}_0 + 
    \frac{\eta L^2}{\mu m} \Omega^{k}_2,\\
    \Omega^{k+1}_1 & \leq (1 - \frac{\delta_c}{2} + \frac{6 \rho^\prime \gamma^2}{\delta_c})\Omega^{k}_1 + 4\rho^\prime L^2 \eta^2 m \Omega^{k}_0 + (6L^2\eta^2 + 2\gamma^2\rho^\prime)\Omega^{k}_2 + 2\eta^2 \Omega^{k}_4,\\
    \Omega^{k+1}_2 & \leq (1-\frac{\gamma\rho}{2})  \Omega^{k}_2 + \frac{2\eta_{out}^2}{\gamma\rho}\Omega^k_4 + \frac{2\gamma\rho^\prime}{\rho} \Omega^k_1, \\
    \Omega^{k+1}_3 & \leq (1 - \frac{\delta_c}{2} + \frac{6 \rho^\prime \gamma^2}{\delta_c})\Omega^{k}_3  +  12\rho^\prime L^2 \eta^2 m \Omega^{k}_0 + 6\delta\rho^\prime L^2 \gamma^2 \Omega^{k}_1 +  (12L^4\eta^2 + 2\gamma^2\rho^\prime L^2)\Omega^{k}_2 + (2\gamma\rho^\prime + 6 L^2\eta^2)\Omega^{k}_4,\\
    \Omega^{k+1}_4 & \leq (1-\frac{\gamma\rho}{2}+\frac{6\eta^2L^2\rho^\prime}{\gamma\rho}) \Omega^{k}_4 + \frac{12m\eta^4L^2}{\gamma\rho} \Omega^{k}_0 + \frac{6L^2\gamma\rho^\prime}{\rho} \Omega^{k}_1 + (\frac{12m\eta^4L^2}{\gamma\rho} + \frac{6L^2\gamma\rho^\prime}{\rho}) \Omega^{k}_2 + \frac{2\gamma\rho^\prime}{\rho} \Omega^{k}_3,\\
    \end{align*}
where $\rho^\prime = \| W - I \|^2 = \sigma_{\text{max}}(W - I)^2$.
\end{lemma}

In contrast to centralized fully first-order bilevel methods that optimize $y$ and $z$ using an averaged $x$,
\begin{align*}
    y_\lambda^\ast(\bar{x}^t) = \arg \min_{y \in \mathbb{R}^{d_y}} \{\frac{1}{m}\sum_{i=1}^m f_i(\bar{x}^t, y) + \lambda g_i(\bar{x}^t, y)\}
\end{align*}

in our case, only local information of $x$ is available. Therefore, the optimal point for our inner loop is given by:
\begin{align*}
    \tilde{y}_\lambda^\ast = \arg \min_{y \in \mathbb{R}^{d_y}} \{\frac{1}{m}\sum_{i=1}^m f_i(x_i^t, y) + \lambda g_i(x_i^t, y)\}
\end{align*}

This results in a deviation of the optimal point with respect to the outer consensus error. To address this, we examine the quantity $|\tilde{y}\lambda^\ast-y\lambda^\ast(\bar{x}^t)|^2$.
\begin{lemma}
\label{lem:inner-loop-opt-deviation}
Let $y_\lambda^\ast(\bar{x}^t)$, $\tilde{y}_\lambda^\ast$, $y^\ast(\bar{x}^t)$, $\tilde{y}^\ast$ defined respectively as
\begin{align*}
    y_\lambda^\ast(\bar{x}^t) &= \arg \min_{y \in \mathbb{R}^{d_y}} \{\frac{1}{m}\sum_{i=1}^m f_i(\bar{x}^t, y) + \lambda g_i(\bar{x}^t, y)\} \\
    \tilde{y}_\lambda^\ast &= \arg \min_{y \in \mathbb{R}^{d_y}} \{\frac{1}{m}\sum_{i=1}^m f_i(x_i^t, y) + \lambda g_i(x_i^t, y)\} \\
    y^\ast(\bar{x}^t) &= \arg \min_{y \in \mathbb{R}^{d_y}} \frac{1}{m}\sum_{i=1}^m g_i(\bar{x}, y) \\
    \tilde{y}^\ast &= \arg \min_{y \in \mathbb{R}^{d_y}} \frac{1}{m}\sum_{i=1}^m g_i(x_i^t, y)
\end{align*}
we can bound the difference between two optima as follows:
\begin{align*}
     \|\tilde{y}_\lambda^\ast-y_\lambda^\ast(\bar{x}^t)\|^2 &\leq \frac{16\kappa^2}{m}\|x-\mathbf{1}\bar{x}^t\|^2
     \\
      \|\tilde{y}^\ast-y^\ast(\bar{x}^t)\|^2 &\leq \frac{\kappa^2}{m}\|x-\mathbf{1}\bar{x}^t\|^2
\end{align*}
\end{lemma}
\begin{proof}
    The defintion of optimal point gives us,
\begin{align*}
    \|\frac{1}{m}\sum_{i=1}^m \nabla h_i (\bar{x}^t , y_\lambda^\ast(\bar{x}^t))\|^2 = 0 , \quad
    \|\frac{1}{m}\sum_{i=1}^m \nabla h_i (x_i , \tilde{y}_\lambda^\ast)\|^2 = 0
\end{align*}
This boils to,
\begin{align*}
    \|\frac{1}{m}\sum_{i=1}^m \nabla h_i (\bar{x}^t , \tilde{y}_\lambda^\ast)\|^2 &=  \|\frac{1}{m}\sum_{i=1}^m \nabla h_i (\bar{x}^t , \tilde{y}_\lambda^\ast) - \nabla h_i (\bar{x}^t , y_\lambda^\ast(\bar{x}^t))\|^2 \geq \frac{\lambda^2\mu^2}{4}\|\tilde{y}_\lambda^\ast - y_\lambda^\ast(\bar{x}^t)\|^2 \\
    \|\frac{1}{m}\sum_{i=1}^m \nabla h_i (\bar{x}^t , \tilde{y}_\lambda^\ast)\|^2 &= \|\frac{1}{m}\sum_{i=1}^m (\nabla h_i (\bar{x}^t , \tilde{y}_\lambda^\ast) - \nabla h_i (x_i , \tilde{y}_\lambda^\ast)) \|^2 \leq \frac{4\lambda^2L^2}{m}\|\mathbf{x} - \mathbf{1}\bar{x}\|^2
\end{align*}
A similar conclusion could be reached on $z$ and this completes the proof.
\end{proof}

\subsection{Lemmas for the Outer Loop}

For the outer loop, we define the following Lyapunov functions:
\begin{equation}
    \begin{aligned}
    & \text{Model consensus error: } \Omega^t_1  := \|\mathbf{x}^t - \textbf{1}\bar{x}^t\|^2, \\
    & \text{Tracker consensus error: } \Omega^t_2  := \|\mathbf{s}_x^t - \textbf{1}\bar{s}_x^t \|^2, \\
    & \text{Value function: } \Omega^t \triangleq \psi(\bar{x}^t) + \frac{1}{m}\Omega^t_1 + \frac{\eta_{out}}{m}\Omega^t_2,
    \end{aligned}
    \label{valuefunc_outerloop}
\end{equation}

In the subsequent sections, we will recursively bound these errors. Before doing so, it is important to address the imprecision in the optimal value between consecutive iterations.

\begin{lemma}
\label{lem:err_2round_innerloop_optimum}
    At the $t$-th iteration, given the optimal value of the inner loop as 
    \begin{align*}
        (\tilde{y}_\lambda^\ast)^t &:= \arg\min_y \left\{\frac{1}{m}\sum_{i=1}^m h_i (x_i^t, y)\right\}, \\ 
        (\tilde{y}^\ast)^t &:= \arg\min_y \left\{\frac{1}{m}\sum_{i=1}^m g_i (x_i^t, y)\right\},
    \end{align*}
    we can bound the difference between two consecutive optima as follows:
    \begin{equation}
        \begin{aligned}
        \|(\tilde{y}_\lambda^\ast)^{t+1} - (\tilde{y}_\lambda^\ast)^{t} \|^2 &\leq \frac{16\kappa^2}{m}\|\mathbf{x}^{t+1} - \mathbf{x}^{t}\|^2, \\
        \| (\tilde{y}^\ast)^{t+1} - (\tilde{y}^\ast)^{t} \|^2 &\leq \frac{\kappa^2}{m} \|\mathbf{x}^{t+1} - \mathbf{x}^{t}\|^2.
        \end{aligned}
    \label{eq:err_opt}
    \end{equation}
\end{lemma}

\begin{proof}
   First, we bound $|(\tilde{y}\lambda^\ast)^{t+1} - (\tilde{y}\lambda^\ast)^t |$. The definition of the optimal point gives us:
    \begin{align*}
        \left\|\frac{1}{m}\sum_{i=1}^m \nabla h_i (x_i^{t+1}, (\tilde{y}_\lambda^\ast)^{t+1})\right\|^2 &= 0, \quad
        \left\|\frac{1}{m}\sum_{i=1}^m \nabla h_i (x_i^t, (\tilde{y}_\lambda^\ast)^{t})\right\|^2 = 0.
    \end{align*}
    Using Proposition \ref{prop:first-order}, we obtain:
    \begin{align*}
        \left\|\frac{1}{m}\sum_{i=1}^m \nabla h_i (x_i^t, (\tilde{y}_\lambda^\ast)^{t+1})\right\|^2 
        &= \left\|\frac{1}{m}\sum_{i=1}^m \left(\nabla h_i (x_i^t, (\tilde{y}_\lambda^\ast)^{t+1}) -  \nabla h_i (x_i^{t+1}, (\tilde{y}_\lambda^\ast)^{t+1})\right)\right\|^2 \\
        &\leq \frac{4\lambda^2L^2}{m}\|\mathbf{x}^{t+1} - \mathbf{x}^{t}\|^2, \\
        \left\|\frac{1}{m}\sum_{i=1}^m \nabla h_i (x_i^t, (\tilde{y}_\lambda^\ast)^{t+1})\right\|^2 
        &= \left\|\frac{1}{m}\sum_{i=1}^m \left(\nabla h_i (x_i^t, (\tilde{y}_\lambda^\ast)^{t+1}) - \nabla h_i (x_i^t, (\tilde{y}_\lambda^\ast)^{t})\right)\right\|^2 \\
        &\geq \frac{\lambda^2\mu^2}{4} \|(\tilde{y}_\lambda^\ast)^{t+1} - (\tilde{y}_\lambda^\ast)^{t}\|^2,
    \end{align*}
    which provides the first result in \eqref{eq:err_opt}.Similarly, we can bound the error for $\tilde{y}^\ast$, completing the proof.
\end{proof}

\begin{lemma}
\label{lemm:outer_recur}
    Under assumptions \ref{assump:smooth}-\ref{assump:graph},let $\Omega^t_1$ and $\Omega^t_2$ be defined as in Eq.\ref{valuefunc_outerloop}. We have the following recursions:
    \begin{align*}
    \Omega^{t+1}_1 & \leq (1-\frac{\gamma\rho}{2})  \Omega^{t}_1 + \frac{6\eta_{out}^2}{\gamma\rho}\Omega^t_2, \\
    \Omega^{t+1}_2 &\leq (1-\frac{\gamma\rho}{2}+\frac{3726(\lambda L_g)^2}{\gamma \rho}\eta^2)\Omega^t_2 + \frac{3726(\lambda L_g)^2}{\gamma \rho}\gamma^2\rho^\prime\Omega^t_1  \\
    & \quad + \frac{3726(\lambda L_g)^2}{\gamma \rho}\eta^2m \|\bar{s}^t \|^2 + 5(\lambda L_g)^2 C_{yz} \alpha^K.
    \end{align*}
\end{lemma}

This lemma indicates that both errors decrease under a proper setting of the hyperparameters.
\begin{proof}
For the first inequality, we start from the update rule for $x$:
\begin{equation}
\begin{aligned}
    \|\mathbf{x}^{t+1} - \textbf{1}\bar{x}^{t+1}\|^2 & = \|\mathbf{x}^t + \gamma(W-I)\mathbf{x}^{t} -\eta \mathbf{s}^t - \textbf{1}\bar{x}^{t} - \eta\textbf{1}\bar{s}^{t} \|^2 \\
    & = \|\widetilde{W}\mathbf{x}^t - \textbf{1}\bar{x}^{k} - \eta(\mathbf{s}^t - \textbf{1}\bar{s}^{t})\|^2 \\
    & \overset{(i)}{\leq} (1+a_1)(1-\gamma\rho)^2\|\mathbf{x}^t - \textbf{1}\bar{x}^{t}\|^2 +  (1+\frac{1}{a_1})2\eta^2\|\mathbf{s}^t - \textbf{1}\bar{s}^{t}\|^2 \\
    & \overset{(ii)}{\leq} (1- \frac{\gamma\rho}{2}) \Omega^{t}_1 + \frac{6\eta^2}{\gamma\rho}\Omega^t_2.
\end{aligned}
\end{equation}
    The inequality $(i)$ use the Young`s inequality and Proposition \ref{prop:spectral_gap_proposition}, $(ii)$ involves the choice of $a_1=\frac{\gamma}{2}$.  
\end{proof}
For $\|\textbf{s}^t - \textbf{1}\bar{s}^t\|^2$, similarly, 
we have,
\begin{equation}
    \begin{aligned}
        \|\textbf{s}^{t+1} - \textbf{1}\bar{s}^{t+1}\|^2 
        &\leq (1-\frac{\gamma\rho}{2}) \|\textbf{s}^t - \textbf{1}\bar{s}^t\|^2 + \frac{3}{\gamma \rho}(3(3\lambda L_g)^2\|\textbf{x}^{t+1} - \textbf{x}^t\|^2 
        \\ \quad &+ 3(2\lambda L_g)^2\|\textbf{y}^{t+1} - \textbf{y}^t\|^2 + 3(L_g)^2\|\textbf{z}^{t+1} - \textbf{z}^t\|^2).
    \end{aligned}
    \label{track_err}
\end{equation}

For $\|\textbf{y}^{t+1} - \textbf{y}^t\|^2$,
\begin{equation}
    \begin{aligned}
        \|\textbf{y}^{t+1} - \textbf{y}^t\|^2 &\leq C_y \alpha^K + 2m\|(\tilde{y}_\lambda^\ast)^{t+1} - (\tilde{y}_\lambda^\ast)^{t}\|^2 \\
        &\leq C_y \alpha^K + 32\kappa^2\|\mathbf{x}^{t+1} - \mathbf{x}^{t}\|^2,
    \end{aligned}
\end{equation}
where the last inequality use lemma \ref{lem:err_2round_innerloop_optimum}. 
Similarly, for $|\mathbf{z}^{t+1} - \mathbf{z}^t|^2$, we have:
\begin{equation}
    \begin{aligned}
         \|\textbf{z}^{t+1} - \textbf{z}^t\|^2 &\leq C_z \alpha^K + 2m\|(\tilde{y}^\ast)^{t+1} - (\tilde{y}^\ast)^{t}\|^2 \\
        &\leq C_z \alpha^K + 2\kappa^2\|\mathbf{x}^{t+1} - \mathbf{x}^{t}\|^2
    \end{aligned}
\end{equation}

The second item could be bounded by
\begin{equation}
    \begin{aligned}
        \|\mathbf{x}^{t+1} - \mathbf{x}^{t}\|^2 &= \|\gamma(W-I)\mathbf{x}^{t} - \eta \mathbf{s}^t\|^2 \\ 
        &=\| \gamma(W-I)(\mathbf{x}^{t} - \mathbf{1}\bar{x}^t) - \eta \mathbf{s}^t\|^2 \\ 
        & \leq 3\gamma^2\rho^\prime\|\mathbf{x}^t - \textbf{1}\bar{x}^t\|^2 + 3\eta^2 \|\mathbf{s}^t - \textbf{1} \bar{s}^t \|^2 + 3\eta^2m \|\bar{s}^t \|^2,
    \end{aligned}
    \label{iter_err2}
\end{equation}
where $\rho^\prime = \| W - I \|^2 = \sigma_{\text{max}}(W - I)^2$.

Combining \eqref{track_err}-\eqref{iter_err2}, 
\begin{equation}
    \begin{aligned}
        \Omega^{t+1}_2 &\leq (1-\frac{\gamma\rho}{2}+\frac{3726(\lambda L_g)^2}{\gamma \rho}\eta^2)\Omega^t_2 + \frac{3726(\lambda L_g)^2}{\gamma \rho}\gamma^2\rho^\prime\Omega^t_1  \\
    & \quad + \frac{3726(\lambda L_g)^2}{\gamma \rho}\eta^2m \|\bar{s}^t \|^2 + 5(\lambda L_g)^2 C_{yz} \alpha^K.
    \end{aligned}
    \label{track_err2}
\end{equation}
This finishes the proof.

\subsection{Proof of Theorem 1}
\begin{theorem}
    \label{inner_loop_convergence}
    If Assumptions \ref{assump:smooth} and \ref{assump:graph} hold, for the inner loop of K steps, at iteration $t$, there exist $\gamma_{in} \propto \delta_c \rho, \eta_{in} \propto \frac{\delta_c \rho^2}{\kappa \lambda L_g}$ such that
    \begin{align*}
        &\|(\mathbf{y}^K)^t - \mathbf{1}\tilde{y}_\lambda^\ast\|^2 \leq C_y \exp(-\frac{\mu K}{8L_g})~~~\text{and}\\
        &\|(\mathbf{z}^K)^t - \mathbf{1}\tilde{y}^\ast\|^2 \leq C_z \exp(-\frac{\mu K}{2L_g}),
    \end{align*}
    where $C_y$ and $C_z$ are positive constants. 
\end{theorem}

Theorem \ref{inner_loop_convergence} is established in Theorem 1 of \citet{9789732}. The proof outlines the existence of a positive number satisfying the inequality $\mathbf{A}\mathbf{\epsilon} \leq (1 - \eta\mu) \mathbf{\epsilon}$, where $\mathbf{\epsilon} = [\epsilon_1, \epsilon_2, L^2 \epsilon_3, \epsilon_4, L^2 \epsilon_5]$ and $\mathbf{A}$ is the matrix of coefficients from Lemma \ref{lem:inner recur}. Assuming that $C_y$ and $C_z$ are the upper bounds of the initial values at iteration $t$, it follows that $\mathbf{y}$ and $\mathbf{z}$ converge to their optimal values at a linear rate.

\subsection{Proof of Theorem 2}
\begin{theorem}
    \label{outer_loop_convergence}
    Under Assumptions \ref{assump:smooth} and \ref{assump:graph}, if we set 
    \begin{align*}
         \lambda\propto \mathcal{O}(l\kappa^3\epsilon^{-1}) \quad \eta_{out} \propto  \mathcal{O}(\gamma l^{-4} \kappa^{-6}\epsilon^{2}) \quad \gamma_{out} \propto  \mathcal{O}(\rho^2),
    \end{align*}
    Algorithm needs $O(l^{4}\kappa^{6}\rho^{-2}\epsilon^{-4}log(\epsilon^{-4}))$ first-order orcale calls to reach the $\epsilon$-stationary point.
\end{theorem}
 Based on above results, we can obtain the final communication complexities of our algorithms $O(l^{4}\kappa^{6}\rho^{-2}\epsilon^{-4}log(\epsilon^{-4}))$.

\begin{proof}
Lets begin with the descent lemma of $\Psi(x)$ under assumption \ref{assump:smooth}, 
\begin{equation}
\begin{aligned}
    \psi(\bar{x}^{t+1}) &\leq \psi(\bar{x}^t) - \nabla \psi(\bar{x}^t)^T(\bar{x}^{t+1}-\bar{x}^t) + \frac{L_\psi}{2}\|\bar{x}^{t+1}-\bar{x}^t\|^2\\
    & = \psi(\bar{x}^t) - \eta_{out}\nabla \psi(\bar{x}^t)^T \bar{s}_x^t + \frac{L_\psi}{2}\|\bar{x}^{t+1}-\bar{x}^t\|\\
    & = \psi(\bar{x}^t) - \frac{\eta_{out}}{2} \|\nabla\psi(\bar{x}^t)\|^2 - (\frac{\eta_{out}}{2}-\frac{L_\psi\eta_{out}^2}{2})\|\bar{s}_x^t\|^2 + \frac{\eta_{out}}{2}\|\nabla\psi(\bar{x}^t)-\bar{s}_x^t\|^2\\
    & \leq \psi(\bar{x}^t) - \frac{\eta_{out}}{2} \|\nabla\psi(\bar{x}^t)\|^2 - (\frac{\eta_{out}}{2}-\frac{L_\psi\eta_{out}^2}{2})\|\bar{s}_x^t\|^2 \\ &\quad + \eta_{out} \|\nabla\psi(\bar{x}^t) - \nabla\psi_\lambda(\bar{x}^t)\|^2 +  \eta_{out} \|\nabla\psi_\lambda(\bar{x}^t)-\frac{1}{m}\sum_{i=1}^{m}\widetilde{\hat{\nabla}_x\psi_i(x_i^t)})\|^2
    \\ &\overset{(i)}{\leq} \psi(\bar{x}^t) - \frac{\eta_{out}}{2} \|\nabla\psi(\bar{x}^t)\|^2 - (\frac{\eta_{out}}{2}-\frac{L_\psi\eta_{out}^2}{2})\|\bar{s}_x^t\|^2 + \mathbf{\textit{O}}(\eta_{out}l^2\kappa^6\lambda^{-2}) \\ &\quad +  \eta_{out} \|\nabla\psi_\lambda(\bar{x}^t)-\frac{1}{m}\sum_{i=1}^{m}\widetilde{\hat{\nabla}_x\psi_i(x_i^t)})\|^2
    \\ &\overset{(ii)} {\leq} \psi(\bar{x}^t) - \frac{\eta_{out}}{2} \|\nabla\psi(\bar{x}^t)\|^2 - (\frac{\eta_{out}}{2}-\frac{L_\psi\eta_{out}^2}{2})\|\bar{s}_x^t\|^2 + \mathbf{\textit{O}}(\eta_{out}l^2\kappa^6\lambda^{-2}) \\ &\quad +  \frac{2\eta_{out}}{m} L_\lambda \|\mathbf{x}^t - \textbf{1}\bar{x}^t\|^2 + 2 \eta_{out} \|\frac{1}{m}\sum_{i=1}^{m} (\nabla\psi_\lambda(x_i^t) -\widetilde{\hat{\nabla}_x\psi_i(x_i^t)}) \|^2, 
\end{aligned}
\label{desl}
\end{equation}
where inequality $(i)$ and $(ii)$ is a result of (1) and (2) from proposition \ref{prop:liptschiz of psi}. 
Note that, 
\begin{equation}
\begin{aligned}
    \nabla\psi_\lambda(x_i^t) &= \nabla_x f_i(x_i,y_\lambda^\ast(\bar{x})) + \lambda\left(\nabla_xg_i(x_i,y_\lambda^\ast(\bar{x}))-\nabla_xg_i(x_i,y^\ast(\bar{x}))\right)\\ 
    \widetilde{\hat{\nabla}_x\psi_i(x_i^t)} & = \nabla_x f_i(x_i,y_i^K)+ \lambda\left(\nabla_x g_i(x_i, y_i^K)- \nabla_x g_i(x_i, z_i^K)\right)
\end{aligned}
\label{eq8}
\end{equation}

\eqref{desl} can be converted to
\begin{equation}
\begin{aligned}
    \psi(\bar{x}^{t+1}) &\leq \psi(\bar{x}^t) - \frac{\eta_{out}}{2} \|\nabla\psi(\bar{x}^t)\|^2 - (\frac{\eta_{out}}{2}-\frac{L_\psi\eta_{out}^2}{2})\|\bar{s}_x^t\|^2 + \mathbf{\textit{O}}(\eta_{out}l^2\kappa^6\lambda^{-2}) \\ &\quad +  \frac{2\eta_{out}}{m} L_\lambda \|\mathbf{x}^t - \textbf{1}\bar{x}^t\|^2 + 2 \eta_{out} \|\frac{1}{m}\sum_{i=1}^{m} (\nabla\psi_\lambda(x_i^t) -\widetilde{\hat{\nabla}_x\psi_i(x_i^t)}) \|^2\\
    &\leq \psi(\bar{x}^t) - \frac{\eta_{out}}{2} \|\nabla\psi(\bar{x}^t)\|^2 - (\frac{\eta_{out}}{2}-\frac{L_\psi\eta_{out}^2}{2})\|\bar{s}_x^t\|^2 + \mathbf{\textit{O}}(\eta_{out}l^2\kappa^6\lambda^{-2}) \\ &\quad +  \frac{2\eta_{out}}{m} L_\lambda \|\mathbf{x}^t - \textbf{1}\bar{x}^t\|^2 + 8\eta_{out} \frac{\lambda^2L_g^2}{m}\|(\textbf{y}^K)^t - \mathbf{1} y_\lambda^\ast(\bar{x}^t)\|^2 + 2\eta_{out} \frac{\lambda^2L_g^2}{m}\|(\textbf{z}^K)^t - \mathbf{1} z_\lambda^\ast(\bar{x}^t)\|^2
\end{aligned}
\label{dstl}
\end{equation}

For the last two item, proposition \ref{lem:inner-loop-opt-deviation} leads to,
\begin{equation}
    \begin{aligned}
        \|(\textbf{y}^K)^t - \mathbf{1} y_\lambda^\ast(\bar{x}^t)\|^2 &\leq 2\|(\textbf{y}^K)^t - \mathbf{1}\tilde{y}_\lambda^\ast\|^2 +32\kappa^2\|x-\mathbf{1}\bar{x}^t\|^2
    \end{aligned}
    \label{optgapy}
\end{equation}

Similarly, for  $\|(\textbf{z}^K)^t - \mathbf{1} z_\lambda^\ast(\bar{x}^t)\|$, 
\begin{equation}
    \begin{aligned}
        \|(\textbf{z}^K)^t - \mathbf{1} z_\lambda^\ast(\bar{x}^t)\|^2 &\leq 2\|(\textbf{z}^K)^t - \mathbf{1}\tilde{z}_\lambda^\ast\|^2 +2\kappa^2\|x-\mathbf{1}\bar{x}^t\|^2
    \end{aligned}
    \label{optgapz}
\end{equation}

Plug \eqref{optgapy} and \eqref{optgapz} into \eqref{dstl}, it comes into
\begin{equation}
    \begin{aligned}
        \psi(\bar{x}^{t+1}) &\leq \psi(\bar{x}^t) - \frac{\eta_{out}}{2} \|\nabla\psi(\bar{x}^t)\|^2 - (\frac{\eta_{out}}{2}-\frac{L_\psi\eta_{out}^2}{2})\|\bar{s}_x^t\|^2 + \mathbf{\textit{O}}(\eta_{out}l^2\kappa^6\lambda^{-2}) \\ &\quad +  \frac{2\eta_{out}}{m} L_\lambda \|\mathbf{x}^t - \textbf{1}\bar{x}^t\|^2 + 8\eta_{out} \frac{\lambda^2L_g^2}{m}\|(\textbf{y}^K)^t - \mathbf{1} y_\lambda^\ast(\bar{x}^t)\|^2 + 2\eta_{out} \frac{\lambda^2L_g^2}{m}\|(\textbf{z}^K)^t - \mathbf{1} z_\lambda^\ast(\bar{x}^t)\|^2\\
        & \leq \psi(\bar{x}^t) - \frac{\eta_{out}}{2} \|\nabla\psi(\bar{x}^t)\|^2 - (\frac{\eta_{out}}{2}-\frac{L_\psi\eta_{out}^2}{2})\|\bar{s}_x^t\|^2 + \mathbf{\textit{O}}(\eta_{out}l^2\kappa^6\lambda^{-2}) \\ &\quad + 262\eta_{out} \frac{\lambda^2L_g\kappa^3}{m} \|\textbf{x}^t- \mathbf{1}\bar{x}^t\|^2 + 16\eta_{out} \frac{\lambda^2L_g^2}{m}\|(\textbf{y}^K)^t - \mathbf{1}\tilde{y}_\lambda^\ast\|^2 + 2\eta_{out} \frac{\lambda^2L_g^2}{m}\|(\textbf{z}^K)^t - \mathbf{1}\tilde{z}_\lambda^\ast\|^2
    \end{aligned}
\end{equation}

Based on Theorem \ref{inner_loop_convergence},combining with lemma \ref{lemm:outer_recur}, now we could reach the point:
\begin{equation}
\label{recur_value}
    \begin{aligned}
         \psi(\bar{x}^{t+1}) &\leq \psi(\bar{x}^t) - \frac{\eta_{out}}{2} \|\nabla\psi(\bar{x}^t)\|^2 - (\frac{\eta_{out}}{2}-\frac{L_\psi\eta_{out}^2}{2})\|\bar{s}_x^t\|^2 + \mathbf{\textit{O}}(\eta_{out}l^2\kappa^6\lambda^{-2}) \\ &\quad + 262\eta_{out} \frac{\lambda^2L_g\kappa^3}{m} \|\textbf{x}^t- \mathbf{1}\bar{x}^t\|^2 + 18\eta_{out} \frac{\lambda^2L_g^2}{m}C_{yz} \alpha^K \\
     \Omega^{t+1}_1 & \leq (1-\frac{\gamma\rho}{2}) \Omega^t_1 + \frac{6\eta_{out}^2}{\gamma\rho} \Omega^t_2 \\
     \Omega^{t+1}_2 &\leq (1-\frac{\gamma\rho}{2}+\frac{3726(\lambda L_g)^2}{\gamma \rho}\eta^2)\Omega^t_2 + \frac{3726(\lambda L_g)^2}{\gamma \rho}\gamma^2\rho^\prime \Omega^t_1 \\ & \quad+ \frac{3726(\lambda L_g)^2}{\gamma \rho}\eta^2m \|\bar{s}^t \|^2 + 5(\lambda L_g)^2 C_{yz} \alpha^K,
    \end{aligned}
\end{equation}
where  $C_{yz} = \max\{C_y, C_z\}$

Finally, given the value function,
\begin{equation}
    \Omega_t = \psi(\bar{x}^t) + \frac{1}{m}\Omega^t_1 + \frac{\eta_{out}}{m}\Omega^t_2,
    \label{valuefunc}
\end{equation}

Combining \eqref{recur_value} and \eqref{valuefunc}, telescope over one iteration,
\begin{equation}
\begin{aligned}
    &\Omega_{t+1} \\
    &\leq  \Omega_{t} - \frac{\eta_{out}}{2} \|\nabla\psi(\bar{x}^t)\|^2 - (\frac{\eta_{out}}{2}-\frac{L_\psi\eta_{out}^2}{2})\|\bar{s}_x^t\|^2 + \mathbf{\textit{O}}(\eta_{out}l^2\kappa^6\lambda^{-2}) \\ &\quad +  262\eta_{out} \frac{\lambda^2L_g\kappa^3}{m} \Omega^t_1 + 18\eta_{out} \frac{\lambda^2L_g^2}{m}C_{yz} \alpha^K \\
    &\quad + \frac{1}{m} \left (-\frac{\gamma\rho}{2} \Omega^t_1 + \frac{6\eta_{out}^2}{\gamma\rho} \Omega^t_2 \right ) \\
    &\quad +  \frac{\eta_{out}}{m} \bigg ((-\frac{\gamma\rho}{2}+\frac{3726(\lambda L_g)^2}{\gamma \rho}\eta^2) \Omega^t_2 + \frac{3726(\lambda L_g)^2}{\gamma \rho}\gamma^2\rho^\prime \Omega^t_1 \\ & \quad+ \frac{3726(\lambda L_g)^2}{\gamma \rho}\eta^2m \|\bar{s}^t \|^2 + 5(\lambda L_g)^2 C_{yz} \alpha^K \bigg ) \\
\end{aligned}
\end{equation}

\begin{align*}
     &\leq \Omega_{t} - \frac{\eta_{out}}{2} \|\nabla\psi(\bar{x}^t)\|^2 - (\frac{\eta_{out}}{2}-\frac{L_\psi\eta_{out}^2}{2} - \frac{3726(\lambda L_g)^2\eta_{out}^2}{\gamma \rho})\|\bar{s}_x^t\|^2 \\ &\quad - \frac{L_g^2}{m}(\frac{\gamma\rho}{2} - 262\eta\lambda^2 \kappa^2 - \frac{\eta\gamma\rho^\prime 3726\lambda^2}{ \rho})\Omega^t_1 \\ &\quad - \frac{\eta_{out}}{m}(\frac{\gamma\rho}{2} - \frac{3726(\lambda L_g)^2}{\gamma \rho}\eta^2 - \frac{6\eta_{out}L_g^2}{\gamma\rho} ) \Omega^t_2 + \frac{23\eta_{out}\lambda^2 L_g^2}{m} C_{yz}\alpha^K + \mathbf{\textit{O}}(\eta_{out}l^2\kappa^6\lambda^{-2})
\end{align*}

if $K \geq \textit{O}(\log(\epsilon^{-4}))$ such that
\begin{equation}
    \frac{23\eta_{out}\lambda^2 L_g^2}{m} C_{yz}\alpha^K \leq \mathbf{\textit{O}}(\epsilon^2)
\end{equation}

Consider the set of parameter
\begin{equation}
    \lambda\propto l\kappa^3\epsilon^{-1} \quad \eta\propto \gamma l^{-4} \kappa^{-6}\epsilon^{2} \quad \gamma\propto \rho^2
\end{equation}
This leads to the following inequality,
\begin{equation}
\begin{aligned}
    \frac{\eta_{out}}{2} \|\nabla\psi(\bar{x}^t)\|^2 &\leq \Omega_{t} - \Omega_{t+1} + \eta_{out}\textit{O}(\epsilon^2)
\end{aligned}
\end{equation}

Telescope over $T$,
\begin{equation}
    \frac{1}{T}\sum_{t=0}^{T-1} \|\nabla\psi(\bar{x}^t)\|^2 \leq \frac{2}{\eta_{out}T}\Omega_{0} + \textit{O}(\epsilon^2)
\end{equation}

Therefore, CDFB needs $O(l^{4}\kappa^{6}\rho^{-2}\epsilon^{-4}log(\epsilon^{-4}))$ first-order orcale calls and $O(l^{4}\kappa^{6}\rho^{-2}\epsilon^{-4}log(\epsilon^{-4}))$ communication rounds to reach the $\epsilon$-stationary point.

\end{proof}

\section{Additional experiments} \label{app:exp}
\subsection{Coefficient Tunning on 20 Newsgroup}
In this experiment, we employ features from the 20 Newsgroup dataset transformed using MinMax scaling and targets for classification. The model operates within a feature space defined by the dataset. Ten clients collaborate across various networks to address the challenges posed by this unique dataset, which is partitioned into both i.i.d. (random split) and non-i.i.d. settings (where h\% portion of data from specific classes is distributed to node $i$).

For the $\text{C}^2$DFB method, we execute 1001 epochs across 15 inner loops. The learning rate for both loops is set at 1, accompanied by a Lagrangian multiplier of 10. The mixing step size is established at 0.5, and compression during communication is predetermined as top-k, retaining 20\% of the parameter portion. For MADSBO, we configure $\alpha = \beta = \gamma = 10$, with 15 iterations per inner loop and a total of 2001 epochs. The moving average constant is set at 0.3. In the case of MDBO, we set $\alpha = \beta_1 = \beta_2 = 5$, $\gamma$ as 0.3 and predefined Lipstchiz continous parameter as 15. 

We conduct a sensitivity analysis on various parameters for this task, as illustrated in Figure \ref{fig:sensitive}. The results indicate that an excessively large number of inner loops and a high sigma value do not lead to a significant increase in performance. This suggests that the $\text{C}^2$DFB algorithm requires only a few inner loops to achieve optimal performance in practical applications.

\begin{figure*}[t]
  \centering
  \includegraphics[width=\linewidth, height=3cm]{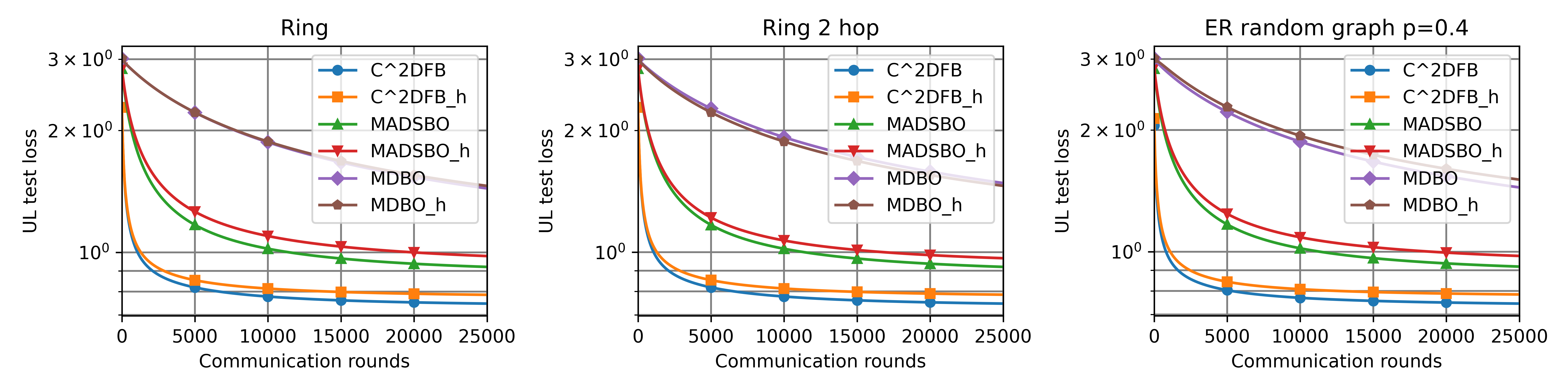}
  \caption{Comparison of test loss against communication round for $\text{C}^2$DFB, MADSBO and MDBO under three topology on Coefficient Tuning task. The 'h' notation represents a heterogeneous data distribution across 10 clients, with a heterogeneity level set to 0.8 in the experiment.}
  \label{fig:lossi2reg}
\end{figure*}

\begin{figure*}[t]
  \centering
  \includegraphics[width=\linewidth, height=3cm]{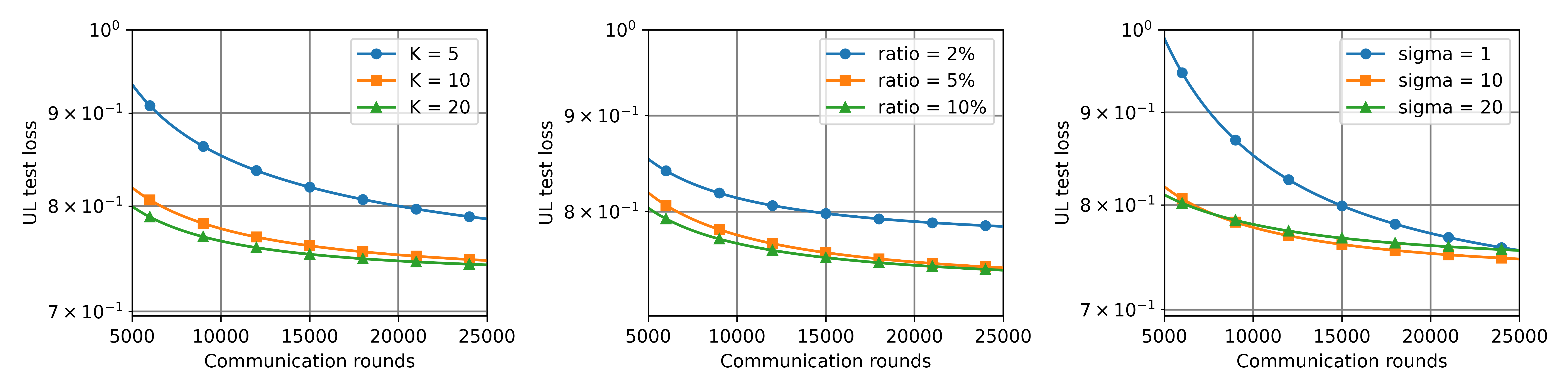}
  \caption{Sensitive studies of $\text{C}^2$DFB, (1) varying the number of inner loops \( K \) (left), (2) varying the compression ratio (middle), and (3) varying the multiplier $\sigma$(right).}
  \label{fig:sensitive}
\end{figure*}

\subsection{Hyper-Representation Learning}
The experiment utilizes the MNIST dataset, where a normalization transform converts images into tensors and standardizes them with a mean of 0.1307 and a standard deviation of 0.3081. We employ a three-layer Multi-Layer Perceptron (MLP) model for image classification. Specifically, the outer optimization targets the hidden units, comprising 81,902 parameters, while the inner optimization focuses on the classification head, which includes approximately 640 parameters. During training, data are randomly split into 10 datasets in the i.i.d. setting and split according to varying probabilities for corresponding classes in the non-i.i.d. setting.

For our $\text{C}^2$DFB approach, the inner learning rate is set at 1 and the outer learning rate at 0.8. Training extends over 80 epochs, with each epoch consisting of 8 iterations for data processing and parameter updates. The mixing step is established at 0.3, and the Lagrangian multiplier is set at 10. During transmission of the inner loop parameters, a compression ratio is maintained at approximately 30\% of the total 640 parameters. For MADSBO, the learning rate is dynamically adjusted during the course of training to enhance performance. Initially, for the first 10 epochs, the learning rate is set at 20 for a warm start, and subsequently adjusted to 100 for the remaining epochs. The outer loop consists of 100 iterations, with 10 inner loops per iteration. The mixing step is consistent with our algorithmic framework. For the naive version of $\text{C}^2$DFB, the hyperparameters are aligned with our $\text{C}^2$DFB to facilitate a direct comparison.

\begin{figure*}[h]
  \centering
  \includegraphics[width=\linewidth, height=3cm]{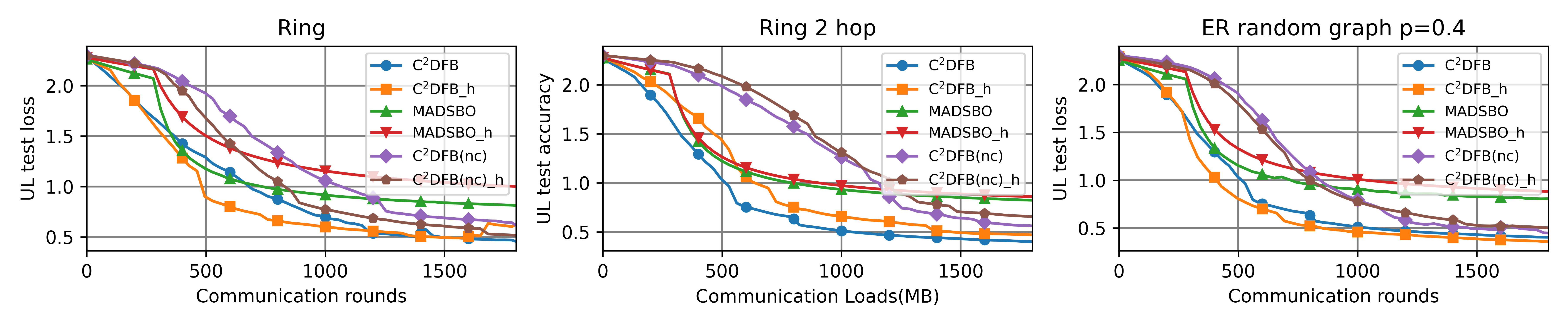}
  \caption{Comparison of test loss against communication round for $\text{C}^2$DFB, MADSBO and  $\text{C}^2$DFB(nc) under three topology on Coefficient Tuning task. The 'h' notation represents a heterogeneous data distribution across 10 clients, with a heterogeneity level set to 0.8 in the experiment.}
  \label{fig:losshyper}
\end{figure*}

% since remaining space is limited, I will put this in the supplementary material

\end{document}